%% file: kernel_appx_paper.tex
\documentclass{article}
\pdfoutput=1
\usepackage{fullpage}
\usepackage[utf8]{inputenc} 
\usepackage[T1]{fontenc}    
\usepackage[hyphens]{url}            
\usepackage{hyperref}       
\usepackage{booktabs}       
\usepackage{amsfonts}       
\usepackage{nicefrac}       
\usepackage{microtype}      
\usepackage{xcolor}         
\usepackage{pgf}
\usepackage{caption}
\usepackage{subcaption}

\input{defs/definitions}

\title{A general technique for approximating high-dimensional empirical kernel matrices}
\author{%
  Chiraag Kaushik\thanks{School of Electrical and Computer Engineering, Georgia Institute of Technology.}
  \and Justin Romberg\footnotemark[1]
  \and Vidya Muthukumar\footnotemark[1] \thanks{School of Industrial \& Systems Engineering, Georgia Institute of Technology.}
  }

\begin{document}

\maketitle
\begin{abstract}
We present simple, user-friendly bounds for the expected operator norm of a random kernel matrix under general conditions on the kernel function $k(\cdot,\cdot)$. Our approach uses decoupling results for U-statistics and the non-commutative Khintchine inequality to obtain upper and lower bounds depending only on scalar statistics of the kernel function and a ``correlation kernel'' matrix corresponding to $k(\cdot,\cdot)$. We then apply our method to provide new, tighter approximations for inner-product kernel matrices on general high-dimensional data, where the sample size and data dimension are polynomially related. Our method obtains simplified proofs of existing results that rely on the moment method and combinatorial arguments while also providing novel approximation results for the case of anisotropic Gaussian data. Finally, using similar techniques to our approximation result, we show a tighter lower bound on the bias of kernel regression with anisotropic Gaussian data.
\end{abstract}

\section{Introduction}
Kernel methods are commonly employed to solve a range of problems in engineering, science, statistics, and machine learning~\cite{scholkopf2018learning}. Traditionally, much of the power of these methods has been derived from the fact that, in classical statistical settings with fixed data dimension, the eigenvalues of the empirical kernel matrix derived from samples behave akin to those of the original kernel integral operator~\cite{koltchinskii2000random}. Together with the universal approximation properties of several common kernels~\cite{steinwart2001influence,micchelli2006universal}, minimax-optimal rates can be derived for a broad class of target functions for kernel ridge regression and the kernel support-vector-machine~\cite{micchelli2006universal,caponnetto2007optimal}. This is no longer the case when kernel methods are applied on high-dimensional data. Indeed, a seminal result of El Karoui \cite{el2010spectrum} showed that in the proportional regime where the number of samples is proportional to the data dimension ($n \asymp d$), a large family of kernel methods equipped with an \emph{inner-product kernel} (i.e., a kernel of the form $k(x,z) = h(\ip{x}{z}/d)$) are restricted in their behavior to linear models. The crux of this result is a proof that the empirical kernel matrix is well-approximated in operator norm by an affine (constant + linear) kernel matrix plus a multiple of the identity matrix. This can be viewed as a type of \textit{curse of dimensionality} for kernel methods, where the nonlinear approximation power becomes negligible when the dimension of the data is proportional to the sample size.

More recently, equivalences between certain types of wide neural networks and kernel methods \cite{jacot2018neural,chizat2019lazy}, the occurrence of phenomena like benign overfitting and double descent in kernel models~\cite{belkin2018understand, haas2023mind, mallinar2022benign, mei2022generalization2, mcrae2022harmless}, and the increasing use of iterative kernel machines \cite{radhakrishnan2024mechanism, zhu2025iteratively} to adapt to hidden low-dimensional structure in modern machine learning tasks have spurred increased interest in sharp analyses of these methods in high-dimensional settings. A flexible but delicate setting that has received substantial recent attention is the \emph{polynomial scaling regime} $n \asymp d^q$, where the number of samples scales as a polynomial power $q \geq 1$ of the data dimension. A reasonable conjecture would be that kernel methods now can approximate only polynomial functions of the data up to degree $\floor{q}$. However, showing this is challenging beyond special cases ($q = 1$ and $q = 2$) and/or specialized assumptions on the data (uniform on the sphere/Boolean hypercube) due to the intricate dependencies between entries of the empirical kernel matrix. Indeed, even the proofs of the results in these specialized settings, e.g.,~\cite{ghorbani2021linearized,pandit2024universality} involve intricate applications of the moment/trace method and combinatorial arguments. The only more general-purpose result shows an approximation barrier of degree-$\floor{2q}$ instead of degree-$\floor{q}$, and the factor of $2$ is not expected to be tight \cite{donhauser2021rotational}.

The central goal of this work is to provide a sharp and general-purpose technique for kernel matrix approximation, applicable in high-dimensional regimes like the above. Specifically, we aim to provide a general technique that can be used to show tight bounds on $\E[||\blK - \bar{\blK}||]$ for any candidate approximator $\bar{\blK}$ — thereby providing conditions under which $\bar{\blK}$ is a faithful approximation of $\blK$ in operator norm. Often, the approximator $\bar{\blK}$ has lower-dimensional structure of some form (e.g., lower-degree, as mentioned above) and is a useful object for studying generalization in kernel ridge regression and its recent iterative extensions \cite{zhu2025iteratively}. In fact, approximation results of this form often constitute the first step in precisely characterizing the test error, which reduces to studying the spectrum of the often simpler matrix $\bar{\blK}$.

\paragraph{Contributions:} In this paper, we first provide a general-purpose bound on the the expected operator norm of an empirical kernel matrix under minimal distributional assumptions and mild integrability conditions on the kernel function $k(\cdot,\cdot)$. We then focus especially on applying our result to inner-product kernels that include those derived as the asymptotic limit of random-feature models/wide neural networks when the number of features/width tends to infinity. For such kernels, our technique recovers the specialized kernel matrix approximation results of~\cite{ghorbani2021linearized,pandit2024universality} in a simpler manner, either matching or improving the best known approximation rates, and significantly improves the approximation barrier from $\floor{2q}$ to $\floor{4q/3}$ on general anisotropic Gaussian data. We finally use similar techniques to obtain new lower bounds for the bias of kernel ridge regression estimates in these high-dimensional settings. While we apply our bound primarily to these types of approximation problems, we have not seen our general kernel matrix bounds stated in this form in the existing literature, and we believe they may find broader use in the analysis of kernel methods on high-dimensional data. Our contributions are listed in more detail below:

\textbf{(1)} 
Our main result, Theorem~\ref{thm:gen-bound}, provides upper and lower bounds on the expected operator norm of a random kernel matrix under general measurability conditions on the kernel function $k$. Our proof relies on a combination of decoupling inequalities for U-statistics and the non-commutative Khintchine inequality and obtains bounds depending on simple scalar statistics of $k$ and a ``correlation kernel'' matrix.

\textbf{(2)} 
We argue that the ``correlation kernel'' matrix appearing in our general bound has a simple form in several common scenarios that arise in the study of high-dimensional kernel regression in the regime $n \asymp d^q$, such as Gegenbauer polynomial, hypercubic Gegenbauer polynomial, and Hermite polynomial kernels. For the already studied cases of data that is uniform on the sphere (corresponding to Gegenbauer polynomial approximation) and uniform on the Boolean hypercube (corresponding to hypercubic Gegenbauer approximation), we recover existing approximation-theoretic results with respect to low-degree polynomial kernel matrices of degree up to $\floor{q}$~\cite{ghorbani2021linearized,mei2022generalization2} as an elementary corollary of Theorem~\ref{thm:gen-bound} (compared to the involved moment/trace method and combinatorial arguments that appear in the proofs of~\cite{ghorbani2021linearized,mei2022generalization2}).

\textbf{(3)} We then turn to the case of anisotropic Gaussian data\footnote{Like~\cite{pandit2024universality}, we can relax the anisotropic Gaussian assumption to a moment-matching assumption, but since the number of moments that would need to be matched will grow with $q$, it would become more stringent.}. We show novel bounds on the approximation error in the scaling regime $n \asymp \tau_1^q$, where $\tau_1 := \tr(\blSigma)$ is a notion of effective dimension.
Here, we show that random inner product kernel matrices can be well-approximated by low degree Hermite polynomial kernel matrices where the degree is upper bounded by $\floor{\frac{4q}{3}}$.
This significantly tightens the polynomial approximation barrier of degree-$\floor{2q}$ for general data under mild bounded-moment assumptions~\cite{donhauser2021rotational}, and also recovers the optimal degree-$2$ approximation barrier recently shown in the \emph{quadratic regime} $q = 2$~\cite{pandit2024universality} with a better approximation error rate.

\textbf{(4)} 
Finally, we show a new lower bound on the bias of kernel ridge (or ridgeless) regression in the case of anisotropic Gaussian data and for a flexible class of target functions that depend on a few scalar projections of the data. This lower bound is also in terms of the best $\floor{\frac{4q}{3}}$-degree approximation to the target function.

\paragraph{Partial progress:} Our results leave open the question of whether the ``polynomial approximation barrier'' can be tightened further from $\floor{\frac{4q}{3}}$ to the conjectured $\floor{q}$ under general anisotropic data.
However, Theorem~\ref{thm:gen-bound}, being an upper bound that is matched in our applications by a lower bound (up to logarithmic factors in $n$), provides valuable insight.
In particular, the lower bound in Corollary~\ref{cor:Hnorm} shows that the approximation barrier cannot be improved beyond $\floor{\frac{4q}{3}}$, \textit{even for isotropic Gaussian data}, if univariate Hermite polynomials are used for the approximation.
This highlights a subtle and fundamental distinction between the utility of using different orthogonal decompositions in kernel matrix approximation and mirrors the key intuition of the recent work~\cite{joshi2025learning}, which also argues that the spherical harmonics (rather than Hermite polynomials) are a more natural univariate basis for analyzing a certain family of single-index models. 
For the isotropic Gaussian case, it is possible to achieve the correct approximation barrier of $\floor{q}$ by using the polar decomposition of a vector $\blx \sim \scrN(\mathbf{0}, \blI_d)$ into independent norm and unit vector terms --- this allows us to approximate the kernel matrix by a degree $\floor{q}$ polynomial of \textit{unit vectors} that are uniformly distributed on the sphere by appealing to the Gegenbauer polynomial expansion (with random coefficients depending on the norms of the data points). 
We formalize this result in Proposition~\ref{prop:Hnorm-isotropic}. While simple, to our knowledge this result has not appeared in the literature as a formal statement. 

As we discuss briefly in Section~\ref{sec:generalK}, this polar decomposition trick can also be applied to anisotropic Gaussian data, but the rescaled vector terms now become anisotropically scaled versions of a uniform distribution on the sphere and are far more complex to deal with.
Finding the right orthogonal basis and decomposition for this case is an important question that we leave open.
However, our results already rule out the Hermite polynomial basis for approximation and provide a flexible testbed for alternatives.

\subsection{Related work}
\paragraph{Decoupling and bounds on random matrices with dependent entries} 
Decoupling inequalities, which aim to reduce stochastic dependencies between random variables, have been developed and applied extensively in the study of U-statistics \cite{de1992decoupling, de2012decoupling} and polynomial chaoses  \cite{kwapien1987decoupling, bandeira2025matrix}. When combined with standard concentration inequalities for sums of independent random variables (like the non-commutative Khintchine (NCK) inequality \cite{chen2012masked, tropp2016expected}), these results have found applications in domains like compressive sensing with structured random matrices \cite{rauhut2010compressive}, learning Gaussian mixtures in high-dimensions \cite{ge2015learning}, and the sum-of-squares algorithm for tensor PCA \cite{bandeira2025matrix}. In this paper, we explore the use of decoupling inequalities and the NCK inequality to bound the expected norm of random empirical kernel matrices, a domain we have not seen previously explored in the literature. Similar to the recent work \cite{tulsiani2024simple}, which studies norms of matrix-valued polynomial chaoses, we find that decoupling leads to much simpler and more generalizable proofs than the popular moment/trace method, which aims to bound $\E \norm{\blK}^{2p} \leq \E[\tr(\blK)^{2p}]$ for some carefully chosen $p$. This type of bound typically requires intricate combinatorial arguments and counting the number of occurrences of different dependency subclasses \cite{ahn2016graph, tulsiani2024simple}. Moreover, the techniques used are often tailored to a specific problem's structure. By contrast, the decoupling approach we use allows for bounds that depend only on simple scalars related to the kernel function and a correlation kernel matrix which we show has a simple form in many applications of interest. 

\paragraph{Empirical kernel matrix approximation:} 
When the data dimension $d$ is held fixed, the classical result~\cite{koltchinskii2000random} shows that the ordered spectrum of the empirical kernel matrix $\blK$ converges to the ordered spectrum of the kernel integral operator as $n \to \infty$ under mild assumptions on the kernel function.
When $d$ grows with $n$, the picture changes considerably.
A complete story has emerged for inner-product kernels of the form $k_d(x,z) = h_d(\ip{x}{z})$ in the \emph{proportional regime} where $d \propto n$.
One line of work considers functions on the inner product scaled as $h_d(z) := h(z/\sqrt{d})$ and precisely characterizes the limiting spectral distribution and/or concentration of the spectral norm of $\blK$ as $d,n \to \infty$~\cite{cheng2013spectrum,do2013spectrum,fan2019spectral}
(these characterizations were also recently extended to the \emph{polynomial regime} where $n \propto d^q$ for some integer $q \in \mathbb{Z}$~\cite{dubova2023universality,lu2025equivalence}).
Interestingly, such a scaling preserves more of the nonlinear information in the function $h(\cdot)$, but does not correspond to the practical kernels arising in machine learning applications, e.g., as the limit of a large number of random features~\cite{rahimi2007random} or neural tangent kernel/lazy training of neural networks~\cite{jacot2018neural,chizat2019lazy}.
Those kernels instead correspond to the scaling $h_d(z) := h(z/\tau_1)$, for which approximation-theoretic characterizations of $\blK$ look very different.
Here, the limiting spectral distribution was provided by~\cite{el2010spectrum,do2013spectrum} and implies that $\blK$ is basically approximated by its entry-wise linearization. The crux of the proof shows that the operator norm of all higher-order terms (i.e., terms of the form $(\blX \blX^T)^{\odot \ell}$ for $\ell \geq 2$) vanishes to $0$\footnote{After this, characterizing the spectrum follows directly from the Marchenko-Pastur law as remaining terms are affine in $\blX \blX^T$.}.
Very high-order terms of the form $\ell \geq 3$ can be handled easily through a Frobenius norm (and therefore entry-wise) upper bound, but the $\ell = 2$ term requires the application of the moment method (to the power $4$) and careful case-by-case analysis of the resultant terms.

The above results are \emph{universal} over data distributions with a bounded $4$th moment.
Unfortunately, they are also pessimistic, as they imply that kernel methods cannot outperform linear models in this regime.
Recent efforts have aimed to characterize the so-called \emph{polynomial regime} where $n \propto d^q$ for some $q > 1$ (which may or may not be integral).
This regime is much more complicated to analyze.
It is possible to show (again, via a Frobenius norm and entry-wise bound) that $\blK$ is well-approximated by the first $\floor{2q}$ terms of the Taylor expansion of $h(\cdot)$ under mild moment assumptions~\cite{donhauser2021rotational} as well as certain fixed-design conditions~\cite{wang2023overparameterized}.
What happens to the ``middle-order'' terms $[\floor{q} + 1,\ldots, \floor{2q}]$ is significantly less clear --- while it is widely believed that $\blK$ will behave like some carefully chosen degree-$\floor{q}$ approximation, this has only been shown for the special cases of data uniformly distributed on the sphere\footnote{Variants of this, such as very specialized ``spiked'' anisotropic distributions on the sphere, have also been analyzed~\cite{ghorbani2020neural}.} or Boolean hypercube~\cite{ghorbani2021linearized,mei2022generalization2}.
This approximation-theoretic characterization is the first step to subsequently sharply characterizing the spectrum of $\blK$ when $q$ is an integer~\cite{hu2022sharp,misiakiewicz2022spectrum}, as well as analyzing the test error of kernel ridge/less regression~\cite{ghorbani2021linearized}.
Instead of a Taylor expansion, these papers expand $h(\cdot)$ in terms of the univariate Gegenbauer polynomial basis, and the main technical result is to show that the operator norms of matrices whose entries comprise higher-order Gegenbauer polynomials vanish.
This result is again shown through the moment method (with a much higher power than $4$ that depends on $q$ and $n$), but handling terms with differing indices is much more challenging than the analysis of~\cite{el2010spectrum}. The authors of~\cite{ghorbani2021linearized} achieved their result through a novel combinatorial ``skeletonization'' technique; namely, repeatedly taking conditional expectations over specific data points and critically relying on an elegant property that the correlation matrix constructed from Gegenbauer polynomial kernels on uniform spherical data is equal to a \emph{scaled-down version of the original Gegenbauer kernel matrix} (see~\eqref{eq:gegenbauer-correlation}).  This technique is involved even for spherical or Boolean data, and it fails to apply in settings where only approximate forms of~\eqref{eq:gegenbauer-correlation} hold, owing to the necessity of taking repeated conditional expectations.
Recently,~\cite{pandit2024universality} provided the correct degree-2 approximation barrier when $q = 2$ for anisotropic Gaussian data (and more generally, data with the first $8$ moments matching those of a multivariate Gaussian).
In this case, the approximation is with respect to matrices whose entries consist of a specific linear combination of univariate Hermite polynomials up to degree $2$. For this result, the authors of 
~\cite{pandit2024universality} also use the moment method and Wick's formula~\cite{wick1950evaluation} (which can be verified to correspond to approximate versions of~\eqref{eq:gegenbauer-correlation}).
By virtue of operating in this \emph{quadratic regime}, they are able to avoid the requirement of repeated ``skeletonizations''; nevertheless, their analysis is still quite involved, and they leave open the question of improving approximation barrier for general polynomial scalings $n \propto \tau_1^q$.

Our decoupling technique is a compelling alternative to the moment method, recovers the results of~\cite{el2010spectrum,ghorbani2021linearized,pandit2024universality} as corollaries, and improves the approximation barrier for the general polynomial regime from the previously known $\floor{2q}$ to $\floor{4q/3}$ under anisotropic Gaussian data.
We provide detailed comparisons/contextualizations with these works throughout the paper.
Because the decoupling technique is matched by lower bounds, we are also able to rule out candidate approximations (e.g. Hermite polynomial approximation for isotropic Gaussian data) and suggest principled alternatives.
Since our bounds are tight, we improve the approximation error rate for anisotropic Gaussian data in the quadratic regime~\cite{pandit2024universality} and match the optimal rate for spherical/Boolean data~\cite{misiakiewicz2024non} up to poly-logarithmic factors in $n$.

\paragraph{Error of kernel ridge/less regression (KRR) in high dimensions:} Recent connections between neural networks and kernel methods~\cite{jacot2018neural,chizat2019lazy,radhakrishnan2024mechanism} and the surprising success of certain interpolating kernels~\cite{belkin2018understand} have spurred intense recent activity on the analysis of kernel ridge/less regression, random feature ensembles and the neural tangent kernel in high dimensions.
We do not survey this literature here (see~\cite{misiakiewicz2024six} for that), but illustrate how kernel matrix approximation is instrumental to sharp characterizations of KRR and its variants.
Traditional analysis of KRR on low-dimensional data shows minimax optimality under general source and capacity conditions (see, e.g.~\cite{caponnetto2007optimal}).
On high-dimensional data, we expect the bias to be a significant factor due to non-trivial approximation error~\cite{belkin2018approximation}.
The optimal approximation bounds on the empirical kernel matrix discussed above were used to sharply characterize the test error of KRR in various high-dimensional regimes through direct bias-variance decompositions involving the empirical kernel matrix~\cite{liang2020just,bartlett2021deep,ghorbani2021linearized,pandit2024universality}.
Examining the proofs of these results reveals that the optimal approximation barrier (i.e. of degree-$\floor{q}$ in the polynomial regime) is essential for the analysis to work.
In settings where optimal approximation results are unavailable, we only have partial characterizations, e.g., lower bounds on the bias~\cite{donhauser2021rotational} or upper bounds on the variance when the target function has bounded Hilbert norm~\cite{liang2020multiple}.
We provide one such partial characterization in the form of a tighter lower bound on the bias of inner-product kernels on general anisotropic Gaussian data (Theorem~\ref{thm:biaslowerbound}).

An alternative approach, that is powerful when we have explicit access to the eigenfunctions and eigenvalues of the kernel integral operator, is to appeal to linear model analysis by showing equivalence to \emph{deterministic error formulas} that depend only on the eigenvalues, or more generally the covariance matrix of an equivalent linear model with Gaussian covariates.
Such equivalences have been established in a general sense for kernels whose eigenfunctions satisfy variants of ``concentration'' properties~\cite{tsigler2023benign,gavrilopoulos2024geometrical,mcrae2022harmless,kaushik2024new,cheng2024dimension}.
The higher-frequency eigenfunctions of inner-product kernels on high-dimensional data (including spherical/Boolean data) can be verified to not satisfy such assumptions, but can be handled separately under a hypercontractivity assumption on only the \emph{low-frequency eigenfunctions}~\cite{mei2022generalization2}. 
Recently,~\cite{misiakiewicz2024non} provided stronger deterministic equivalence results under weaker assumptions and unified most of the above cases.
All of the results above importantly rely on being able to approximate the ``higher-frequency'' part of the empirical kernel matrix by a multiple of the identity.
This is related in spirit (but not identical) to the empirical kernel matrix approximations that we study.
At a higher level, all of these results require access to the eigenfunctions and eigenvalues, which is an independent challenge for practical inner-product kernels (outside the special case of data that are uniformly distributed on the sphere or Boolean hypercube~\cite{ghorbani2021linearized,mei2022generalization2}).


\subsection{Notation}
We use lowercase boldface characters (e.g., $\blx$) for vectors and uppercase boldface characters (e.g., $\blX$) for matrices. Since our main result, Theorem~\ref{thm:gen-bound}, could apply to generic data, we do not use this convention there and simply refer to data as, e.g., $x$. $\blI_k$ denotes the identity matrix of dimension $k$.
The notation $\diag(\blA)$ and $\diag^\perp(\blA)$ denotes the diagonal and off-diagonal parts, respectively, of a square matrix $\blA$. The symbol $\mathbf{1}_k$ denotes the ones vector of dimension $k$. We use $\|\cdot \|$ to denote the operator norm in the case of a matrix, and $\|\cdot\|_F$ to denote its Frobenius norm. For vectors, $\| \cdot \|$ and $\| \cdot \|_2$ denote the Euclidean norm.
The inequality $x \lesssim y$ will be used to refer to $x \leq C y$ for a sufficiently large universal constant $C > 0$; we have $x \asymp y$ iff $x \lesssim y$ and $y \lesssim x$. Similarly, we use the notation $x \lesssim_{\log} y$  (resp. $x \gtrsim_{\log} y$) to indicate $x\leq Cy \log^{c}(n)$ (resp. $x\geq Cy \log^{c}(n)$), for sufficiently large universal constants $C, c > 0$. Universal constants in general can change line to line. We use the notation $o_\tau(1)$ to indicate quantities that decay to $0$ in the limit as $\tau \to \infty$.

The $\ell$-th order derivative of any $\ell$-times differentiable function $f \colon \R \to \R$ is denoted by $f^{(\ell)}(\cdot)$. A function is said to be in $C^{(k)}$ if it is $k$-times continuously differentiable. We let $\He_\ell(\cdot)$ denote the $\ell$-th (probabilist's) Hermite polynomial. We denote the sphere of radius $r$ in $\R^d$ as $\scrS^{d-1}(r)$. When $x$ and $y$ are independent random variables, we use the notation $\E_{x}[f(x,y)]$ to denote the conditional expectation $\E[f(x,y)|y]$.

\section{Main result}
In this section, we develop a general bound for the expected operator norm of random kernel matrices. Let $x_1, \dots, x_n$ be independent variables in a probability space $(\scrX, \scrP)$, and let $k \colon \scrX \times \scrX \to \R$ be a positive semi-definite kernel function satisfying $\E\abs{k(x_1, x_2)} < \infty$. Define the kernel matrix $\blK \in \R^{n \times n}$ with $\blK_{ij} = k(x_i, x_j)$. The following theorem gives an upper bound on the expected operator norm of $\blK$.

\begin{theorem}[General kernel matrix upper bound]\label{thm:gen-bound}
Let $z, x_1, x_2, \dots, x_n ~\simiid ~\scrP$. Then, we have
\begin{align*}
    \E \norm{\blK} &\lesssim \E \max_{1 \leq i \leq n} |k(x_i, x_i)| + n\sqrt{\log{n}\E[\E_z[k(x_1,z)]^2]} \\
    &+ \sqrt{n\log{n}\E\norm{\blG}}
    +\log{n} \sqrt{n \E \brackets*{\max_{1\leq i \leq n} \parens{k(z, x_i) - \E_{x_i} k(z, x_i)}^2}},
\end{align*}
where $\blG \in \R^{n\times n}$ is the correlation matrix with entries given by $G_{ij} = \E_{z}[k(x_i, z)k(z,x_j)]$.
\end{theorem}
Before proceeding with the proof, we note that this result obtains an upper bound in terms of relatively simple scalar quantities related to the statistics of $k$ (which can often be computed easily using properties of the data distribution) and the correlation matrix $\blG$. For many kernels of interest, as we will see in the following section, the correlation matrix term $\E \norm{\blG}$ can be bounded either in terms of $\E\norm{\blK}$ itself or through a simple Frobenius norm upper bound.
\begin{proof}

First, by separating the diagonal and off-diagonal parts of $\blK$ we obtain the simple bound
\[
\E \norm{\blK} \leq \E \max_{i} |k(x_i, x_i)| + \E \norm{\blDelta},
\]
where we define $\blDelta \coloneqq \diag^{\perp}\blK$ as the off-diagonal component of $\blK$. 

The first step is to relate the operator norm of $\blDelta$ to a certain ``decoupled'' matrix with independent columns.
In particular, note that we can write
\begin{align*}
    \blDelta = \sum_{j=1}^n \sum_{i \neq j} \frac{k(x_i, x_j)}{2}(\bm{e}_i \bm{e}_j^\top + \bm{e}_j \bm{e}_i^\top).
\end{align*}
Observe that we can express this in the form of a U-statistic $\sum_{1 \leq i \neq j \leq n} f_{ij}(\blx_i, \blx_j)$ where $f_{ij}(\blx_i, \blx_j) := \frac{k(x_i, x_j)}{2}(\bm{e}_i \bm{e}_j^\top + \bm{e}_j \bm{e}_i^\top)$.
Here, the range space of each $f_{ij}$ is the matrix-valued Banach space endowed with the $\|\cdot\|$ operator norm, and $f$ is Bochner-integrable by our integrability assumption on $k$.

Then, define the decoupled matrix
\begin{align*}
    \widetilde{\blDelta} := \sum_{j=1}^n \sum_{i \neq j} \frac{k(x_i, \xtl_j)}{2}(\bm{e}_i \bm{e}_j^\top + \bm{e}_j \bm{e}_i^\top),
\end{align*}
where $(\xtl_1,\ldots,\xtl_n)$ is an i.i.d. copy of $(x_1,\ldots,x_n)$.
A direct application of the decoupling inequality (Theorem 1 in~\cite{de1992decoupling}) gives us
\begin{align}\label{eq:decoupling}
    \E\left[\|\blDelta\|\right] \leq 8 \cdot \E \left[ \|\widetilde{\blDelta} \| \right],
\end{align}
where the latter expectation is taken over both $(x_1, \dots, x_n)$ and $(\xtl_1, \dots, \xtl_n)$.
We will upper bound the RHS of Equation~\eqref{eq:decoupling} by noting that we can write $\widetilde{\blDelta}$ as a sum of random matrices that are \emph{independent} conditioned on $(x_1,\ldots, x_n)$.
In particular, using the tower property of conditional expectations, we have

\begin{align*}
    \E\left[ \|\widetilde{\blDelta} \|\right] &= \E\brackets*{\E\brackets*{\norm*{\sum_{j=1}^n \blZ_j}\Big{|} (x_1, \ldots, x_n)}},\\
\end{align*}
where we have defined $\blZ_j \coloneqq \sum_{i \neq j} \frac{k(x_i, \xtl_j)}{2}(\ble_i \ble_j^\top + \ble_j\ble_i^\top)$. Observe that, conditioned on $(x_1,\ldots,x_n)$, the random matrices $\blZ_j$ are independent. 

Next, we will use the general-purpose non-commutative Khintchine inequality for a sum of independent random matrices \cite[Theorem A.1]{chen2012masked} (see also 
\cite{tropp2016expected} for a simple proof) to characterize the operator norm of $\widetilde{\blDelta} = \sum_{j=1}^n \bm{Z}_{j}$.
In particular, we have
\begin{align}\label{eq:conditional-bound}
    \E\left[\|\widetilde{\blDelta}\| \Big{|} (x_1,\ldots,x_n)\right] \lesssim \norm*{\E \brackets*{\widetilde{\blDelta} \big{|} (x_1, \dots, x_n)}} + \sqrt{\log(n) \cdot V} + \log(n) \cdot L,
\end{align}
where we define
\begin{align*}
    V  &= \norm*{\sum_{j=1}^n \E\left[(\blZ_{j} - \E [\blZ_{j} \big{|} (x_1, \dots, x_n)])^2 \Big{|} (x_1, \dots, x_n)\right]} \text{ and } \\
    L^2 &= \E \left[ \max_j \norm*{\blZ_{j} - \E [\blZ_{j} \big{|} (x_1, \dots, x_n)]}^2 \Big{|} (x_1, \dots, x_n) \right].
\end{align*}
Consequently, to upper bound the original quantity of interest, $\E\left[\|\blDelta\|\right]$, it now suffices to upper bound the expectation of the RHS of Equation~\eqref{eq:conditional-bound} over the original data $(x_1,\ldots, x_n)$. To this end, we individually characterize each term appearing in Equation \eqref{eq:conditional-bound}. In the remainder of the proof, we will use the shorthand $\E_{\xtl}[\cdot] \coloneqq \E_{(\xtl_1,\ldots,\xtl_n)}[\cdot]$ for brevity.

\paragraph{Bounding the norm of the expected matrix: }
For convenience, define the function $h(x) = \E[k(x, \xtl) \given x]$. Then, this term can be written as 
\begin{align*}
    \E\norm*{\E_{\xtl} \brackets*{\widetilde{\blDelta}}} &= \frac{1}{2} \E [\norm{\blH + \blH^\top}], 
\end{align*}
where $H_{ij} = h(x_i)$. By the triangle inequality, we have
\begin{align*}
    \E\norm*{\E_{\xtl} \brackets*{\widetilde{\blDelta}}} &\leq \E [\norm{\blH}]\\
    &= \E \norm*{(h(x_1), \dots, h(x_n))^\top \mathbf{1}_n^\top}\\
    &\leq \sqrt{n \sum_{j=1}^n \E h(x_j)^2}\\
    &= n \sqrt{\E h(x_1)^2} = n\sqrt{\E[\E_z[k(x_1,z)]^2]},
\end{align*}
where the second-to-last line follows from Jensen's inequality. 
\paragraph{Bounding $L$:} 
Note that for any $j$, the matrix $\blZ_{j} - \E \blZ_{j}$ is symmetric and consists of a single non-zero row and column. Hence, we can upper bound the operator norm by the Frobenius norm to obtain 
\begin{align*}
\norm*{\blZ_{j} - \E_{\xtl} \blZ_{j}}^2 &\leq \norm*{\blZ_{j} - \E_{\xtl} \blZ_{j}}^2_F\\
&=2 \sum_{i\neq j} \frac{1}{4} (k(x_i, \xtl_j)-\E_{\xtl}[k(x_i, \xtl_j)])^2\\
&\lesssim \sum_{i =1}^n (k(x_i, \xtl_j)-\E_{\xtl}[k(x_i, \xtl_j)])^2.
\end{align*}
Substituting into the expression for $L$, we arrive at
\begin{align*}
    L^2 &\leq \E_{\xtl} \max_{1\leq j \leq n} \sum_{i =1}^n (k(x_i, \xtl_j)-\E_{\xtl}[k(x_i, \xtl_j)])^2\\
    &\leq  \sum_{i=1}^n \E_{\xtl} \max_{1 \leq j \leq n} (k(x_i, \xtl_j)-\E_{\xtl}[k(x_i, \xtl_j)])^2.
\end{align*}
So, taking the expectation over $(x_1, \dots, x_n)$ and applying Jensen's inequality, we obtain
\begin{align*}
    \E L &\leq \sqrt{\E L^2}\\
    &\leq \sqrt{\sum_{i=1}^n \E \max_j \brackets*{k(x_i, \xtl_j)- \E_{\xtl}[k(x_i, \xtl_j)]}^2}\\
    &= \sqrt{n} \sqrt{\E_{(z, x_1, \dots, x_n)} \max_{1\leq i \leq n} \parens{k(z, x_i) - \E[k(z, x_i)\given z]}^2}.
\end{align*}

\paragraph{Bounding $V$:} For simplicity of notation, let $\bar{k}(x_1, x_2) := k(x_1,x_2) - \E_{x_2} k(x_1, x_2)$ be the centered kernel (with respect to the second input $x_2$). We also define the matrix $\bar{\blG} \in \R^{n \times n}$ to have entries given by $\bar{\blG}_{i,i'} = \E_z[\bar{k}(x_i, z)\bar{k}(x_{i'}, z)]$. Note that $\bar{\blG}$ is itself a PSD matrix, since it is a Gram matrix with entries given by inner products in $L^2(\scrP)$. We will also write $\bar{\blG}_{\setminus j} \in \R^{n \times n}$ to denote the ``leave-one-out'' versions of $\bar{\blG}$, where the $j$-th row and column are set to $0$. 

With this notation in hand, we can compute 
\begin{align*}
    \E_{\xtl} [(\blZ_j - \E_{\xtl}\blZ_j)^2] &= \E_{\xtl}\parens*{\sum_{i \neq j} \frac{1}{2} \bar{k}(x_i, \xtl_j)(\ble_i \ble_j + \ble_j \ble_i)^\top}^2\\
    &= \frac{1}{4} \sum_{i, i' \neq j} \E_{\xtl_j}[\bar{k}(x_i, \xtl_j)\bar{k}(x_{i'}, \xtl_j)](\ble_i \ble_{i'}^\top + \delta_{i,i'} \ble_j \ble_j^\top)\\
    &= \frac{1}{4} \parens*{\bar{\blG}_{\setminus j} + \sum_{i \neq j} \bar{\blG}_{ii} \ble_j \ble_j^\top}.
\end{align*}
Therefore, we have
\begin{equation}\label{eq:Vbound}
V = \frac{1}{4}\norm*{\sum_{j=1}^n \parens*{\bar{\blG}_{\setminus j} + \sum_{i \neq j} \bar{\blG}_{ii} \ble_j \ble_j^\top}}.
\end{equation}
Applying the triangle inequality, we obtain
\begin{align*}
V &\lesssim \norm*{\sum_{j=1}^n \bar{\blG}_{\setminus j}} + \norm*{\sum_{j=1}^n \sum_{i\neq j} \bar{\blG}_{ii} \ble_j \ble_j^\top}\\
&= \norm*{\bar{\blG} \odot ((n-2)\mathbf{1}_n \mathbf{1}_n^\top + \blI_n)} + \max_j \sum_{i\neq j} \bar{\blG}_{ii}\\
&\overset{(1)}\leq n \norm{\bar{\blG}} + \tr{\bar{\blG}} \lesssim n\norm{\bar{\blG}},
\end{align*}
where inequality (1) uses the triangle inequality. Finally, taking the expectation over $(x_1, \dots, x_n)$ and using Jensen's inequality, we arrive at
\begin{align*}
    \E_{(x_1, \dots, x_n)} \sqrt{V} \leq \sqrt{n \E\norm{\bar{\blG}}}.
\end{align*}
To convert this into a bound in terms of the uncentered correlation kernel matrix $\blG$ (corresponding to the original kernel $k$), note that $\bar{\blG} = \blG - \blG'$, where $G'_{i,j} = \E_z[k(x_i, z)]\E_z[k(x_j, z)]$. Therefore, we have
\begin{align*}
    \E\norm{\bar{\blG}} &\leq \E\norm{\blG} + \E\norm{\blG'}.
\end{align*}
We bound the second term using a Frobenius norm upper bound as below:
\begin{align*}
\norm{\blG'} &\leq \sqrt{\sum_{i, j = 1}^n \E_{z} \left[ k(x_{i}, z)  \right]^2 \E_{z}\left[k(x_{j}, z) \right]^2}.
\end{align*}
Taking the expectation over $(x_1, \dots, x_n)$ and again using Jensen's inequality, we obtain
\[
\E_{(x_1, \dots, x_n)} \norm{\blG'} \leq  \sqrt{\sum_{i,j = 1}^n (\E_{x_i} (\E_{z}\left[k(x_{i}, z) \right])^2)^2} \leq n \E_{x_1}(\E_z k(x_1, z))^2.
\]
Combining the above, we obtain the final bound on $V$:
\[
\E_{(x_1, \dots, x_n)} \sqrt{V} \leq \sqrt{n\E\norm{\blG}} + n \sqrt{\E_{x_1}(\E_z k(x_1, z))^2}.
\]
Substituting each of these 3 bounds into Equation \eqref{eq:conditional-bound} completes the proof of the theorem. 
\end{proof}
\subsection{Lower bound}
We next note that the key steps of the above proof (namely, decoupling and the application of the non-commutative Khintchine inequality) also have matching lower bounds, so we can also obtain a similar lower bound on $\E\norm{\blK}$. The lower bound we state holds for general kernel functions with conditionally zero mean, i.e., where the expectation when conditioning on one input is zero. We note that we expect the dominant term in the lower bound to be $\E\sqrt{n\norm{\blG}}$, which matches the expression obtained in the proof of Theorem \ref{thm:gen-bound} (in the proof of Theorem \ref{thm:gen-bound} we further upper bound this using Jensen's inequality, yielding the term $\sqrt{n \E \norm{\blG}}$ --- we did this for convenient usage in applications to follow).

\begin{theorem}[General kernel matrix lower bound]\label{thm:lower-bound}
Assume the kernel function $k$ additionally satisfies $\E[k(x,z)\given z] = 0$. Then, we have
\begin{align*}
    \E \norm{\diag^{\perp}(\blK)} &\gtrsim \max\braces*{\E\sqrt{n\norm{\blG}}, \E \sqrt{\sum_{i>1} G_{ii}}},
\end{align*}
where $\blG \in \R^{n\times n}$ is the correlation kernel matrix defined in Theorem~\ref{thm:gen-bound}.
\end{theorem}
\begin{proof}
Again, directly applying the decoupling inequality (Theorem 1 in~\cite{de1992decoupling}) --- after noting that the corresponding $f_{ij}$'s are symmetric --- gives us
\begin{align}\label{eq:decoupling-lowerbound}
    \E\left[\|\blDelta\|\right] \geq \frac{1}{4} \cdot \E \left[ \|\widetilde{\blDelta} \| \right],
\end{align}
with $\widetilde{\blDelta}$ defined as in the proof of Theorem \ref{thm:gen-bound}. Again conditioning on $(x_1, \dots, x_n)$ and applying Theorem 1 of \cite{tropp2016expected}, we obtain 
\begin{align*}
    \E\left[ \|\widetilde{\blDelta} \|\right] &= \E\brackets*{\E\brackets*{\norm*{\sum_{j=1}^n \blZ_j}\Big{|} (x_1, \ldots, x_n)}} \gtrsim \E \sqrt{V}.\\
\end{align*}
Substituting the expression for $V$ computed in Equation~\eqref{eq:Vbound}, we have
\begin{align*}
\E\sqrt{V} &= \E \sqrt{\norm*{\sum_{j=1}^n \blG_{\setminus j} + \sum_{j=1}^n \sum_{i \neq j} G_{ii} \ble_j \ble_j^\top}}\\
&\geq \max \braces*{ \E \sqrt{\norm*{\sum_{j=1}^n \blG_{\setminus j}}}, \E \sqrt{\norm*{\sum_{j=1}^n \sum_{i \neq j} G_{ii} \ble_j \ble_j^\top}}},\\
\end{align*}
where we use the fact that $\norm{\blA + \blB} \geq \max\{\norm{\blA}, \norm{\blB}\}$ for two PSD matrices $\blA$ and $\blB$. The first term can be written as
\begin{align*}
\E \sqrt{\norm*{\blG \odot ((n-2)\mathbf{1}_n \mathbf{1}_n^\top + \blI_n)}} &= \E \sqrt{\norm*{(n-2)\blG + \diag(\blG)}} \gtrsim \E \sqrt{n \norm{\blG}},
\end{align*}
and the second term can be bounded below by
\begin{align*}
\E \sqrt{\norm*{\sum_{j=1}^n \sum_{i \neq j} G_{ii} \ble_j \ble_j^\top}} &\geq \E \sqrt{\sum_{i>1} G_{ii}}.
\end{align*}
This concludes the proof of the lower bound.
\end{proof}
\section{Applications of Theorem~\ref{thm:gen-bound}}
In this section, we apply our general theorem to a few situations of interest that arise in the study of high-dimensional kernel regression equipped with inner-product kernels.
As mentioned in the introduction, in all of the applications we will use the theorem to bound the error of the original empirical kernel matrix with respect to a suitable low-degree approximation. Before proceeding, we note that although we state these corollaries as bounds on the expected norm, a simple application of Markov's inequality can be used to convert our results to high-probability bounds. For example, we can conclude that the same upper bounds with an additional multiplicative factor of $\log{n}$ hold with probability at least $1-\frac{1}{\log{n}}$.

\subsection{Gegenbauer polynomial kernels}

We first consider \emph{Gegenbauer polynomial kernels} which arise naturally in the analysis of inner-product kernels on data that are uniformly distributed on the sphere~\cite{ghorbani2021linearized,mei2022generalization2}.
In this case, we will show that Theorem~\ref{thm:gen-bound} directly gives the approximation-theoretic bounds proved in~\cite{ghorbani2021linearized,mei2022generalization2} as a simple corollary.

Let $\blx_1, \dots, \blx_n$ be independent and uniformly distributed on $\scrS^{d-1}(\sqrt{d})$, and let $k(\blx, \bly) = Q_\ell^{(d)}(\ip{\blx}{\bly})$, where $Q_{\ell}^{(d)} \colon [-\sqrt{d}, \sqrt{d}] \to \R$ is the $\ell$-th Gegenbauer polynomial. These polynomials form an orthogonal basis for the space $L^2([-\sqrt{d}, \sqrt{d}],\tilde{\tau}_{d-1})$ where $\tilde{\tau}_{d-1}$ is the distribution of $\sqrt{d}\ip{\blx}{\ble_1}$ when $\blx \sim \text{Unif}(\scrS^{d-1}(\sqrt{d}))$. We follow the normalization convention in \cite{ghorbani2021linearized}, which is restated below:
\begin{equation}\label{eq:gegenbauer-norm}
\ip{Q_k^{(d)}}{Q_{\ell}^{(d)}}_{L^2(\tilde{\tau}_{d-1})} = \frac{1}{B(d, \ell)} \delta_{k\ell} ,
\end{equation}
where $B(d,\ell) \asymp d^\ell$ denotes the number of spherical harmonics of degree $\ell$ in $d$ dimensions. Under this scaling, we also have $Q_\ell^{(d)}(d) = 1$ and 
the crucial property that for $\blx, \blz \in \scrS^{d-1}(\sqrt{d})$ and $\bly \sim \text{Unif}(\scrS^{d-1}(\sqrt{d}))$,
\begin{align}\label{eq:gegenbauer-correlation}
\E_{\bly} \left[Q_k^{(d)}(\ip{\blx}{\bly})Q_\ell^{(d)}(\ip{\bly}{\blz})\right] = \frac{1}{B(d,\ell)} Q_\ell^{(d)}(\ip{\blx}{\blz}) \delta_{k\ell}.
\end{align}
Note that Equation~\eqref{eq:gegenbauer-correlation} essentially implies that the correlation matrix of a Gegenbauer polynomial kernel is equal to the original Gegenbauer kernel matrix scaled down by the factor $B(d,\ell)$. We refer the reader to \cite{ghorbani2021linearized} for further background on these polynomials.

Consider the off-diagonal component of the Gegenbauer polynomial kernel matrix, denoted by $\blDelta^{(\ell)}$, with entries
\[
\blDelta^{(\ell)}_{ij} =  Q_k(\ip{\blx_i}{\blx_j})\mathbf{1}\{i \neq j\}.
\]

Applying Theorem \ref{thm:gen-bound}, we obtain the following corollary:
\begin{corollary}\label{cor:Gnorm}
    For the Gegenbauer polynomial matrix described above and $\ell > 0$, 
    \[
    \sqrt{nd^{-\ell}} \lesssim \E[\norm{\blDelta^{(\ell)}}] \lesssim_{\log} nd^{-\ell} + \sqrt{nd^{-\ell}}.
    \]
    In particular, if $n \asymp d^q$ for $q < \ell$, then $ \E[\norm{\blDelta^{(\ell)}}] = o_d(1)$.
\end{corollary}
Before proceeding to the proof, we note that this corollary recovers the main result of \cite[Proposition 3]{ghorbani2021linearized} and \cite[Proposition 13]{misiakiewicz2024non} via a much simpler argument that does not rely on the moment method and involved combinatorial calculations. This result is important in the analysis of kernel regression with uniform spherical data and can be used to show that, in the polynomial scaling regime $n\asymp d^q$, kernel regression estimates for a wide class of inner product kernels behave like low-degree polynomial kernels up to degree \emph{exactly equal to} $\floor{q}$. 
This, in turn, facilitates a sharp analysis of kernel ridge/ridgeless regression --- see~\cite{ghorbani2021linearized} for such a full analysis, and also the subsequent works~\cite{mei2022generalization2, misiakiewicz2024non}.
We additionally note that Corollary~\ref{cor:Gnorm} recovers the optimal rate of \cite[Proposition 13]{misiakiewicz2024non} up to a poly-logarithmic factor in $n$.

\begin{proof}[Proof of Corollary \ref{cor:Gnorm}]
Applying Theorem \ref{thm:gen-bound} and the property of Gegenbauer polynomials in Equation~\eqref{eq:gegenbauer-correlation}, we obtain (suppressing log factors):
\[
\E[\norm{\blDelta^{(\ell)}}] \lesssim_{\log} \sqrt{nd^{-\ell} (\E\norm{\blDelta^{(\ell)}} + 1)} +  \sqrt{n \E \max_{1\leq i \leq n} \parens{k(\blz, \blx_i)}^2}.
\]
For the latter term, note that
    \begin{align*}
         \E \max_{1\leq i \leq n} \parens{k(\blz, \blx_i)}^2 &= \E \max_i \brackets*{Q_\ell\parens*{\ip{\blx_i}{\blz}}}^2\\
        &\lesssim \log^c(n) \cdot \E_{\blz}  \norm*{Q_{\ell}(\ip{\cdot}{\blz})}_{L^2}^2,
    \end{align*}
where we apply Lemma \ref{lem:maxpoly} conditionally on $\blz$. We can bound the $L^2$ norm using Equation \eqref{eq:gegenbauer-correlation} by noting that (with the expectation conditional on $\blz$)
\begin{align*}
   \norm*{Q_{\ell}(\ip{\cdot}{\blz})}_{L^2}^2 &= \E_{\blx} \brackets*{Q_{\ell}(\ip{\blx}{\blz})^2} \asymp d^{-\ell}.
\end{align*}
Substituting this into the bound above we obtain
\[
\E[\norm{\blDelta^{(\ell)}}] \lesssim_{\log} \sqrt{nd^{-\ell} (\E\norm{\blDelta^{(\ell)}} + 1)} +  \sqrt{nd^{-\ell}},
\]
which implies the stated upper bound. For the lower bound, we apply Theorem \ref{thm:lower-bound} and only take the second term in the maximum to conclude that
\[
\E[\norm{\blDelta^{(\ell)}}] \gtrsim \E \sqrt{\sum_{i>1} d^{-\ell}} \asymp \sqrt{nd^{-\ell}}.
\]
\end{proof}
A similar result also holds for the ``hypercubic Gegenbauer'' polynomials studied in \cite{mei2022generalization2} and uniform data on the binary hypercube; the argument is identical, so we do not include it here.

\subsection{Hermite polynomial kernels}\label{sec:hermite}

In this subsection and the next section, we turn to the more difficult problem of approximating an empirical kernel matrix whose entries consist of inner-product kernel evaluations on high-dimensional, anisotropic Gaussian data. 
Recall that the only tight approximations known in this case were obtained for the linear regime $d \propto n$ (corresponding to $q = 1$)~\cite{el2010spectrum} and the quadratic regime $d \propto n^2$ (corresponding to $q = 2$)~\cite{pandit2024universality}. 
Ultimately, we will present improved approximation results for the general polynomial regime $d \propto n^q$, where $q$ may or may not be an integer.
A natural candidate for polynomial approximation, as put forward by~\cite{pandit2024universality}, would be the univariate Hermite polynomials.
Formally, let $\blx_1, \dots, \blx_n \simiid \scrN(\mathbf{0}, \blSigma)$. Denote $\tau_k \coloneqq \tr(\blSigma^k)$ and $R = \frac{\tau_2^2}{\tau_4}$. The quantities $\tau_k$ and $R$ can be considered notions of \textit{effective dimension} that all reduce to $d$ when $\blSigma = \blI_d$; we will see that the bounds we obtain depend on these quantities in a nuanced way. Consider the off-diagonal component of the Hermite polynomial kernel matrix $\blDelta^{(\ell)}$ given by
\[
\blDelta^{(\ell)}_{ij} \coloneqq \He_{\ell}\parens*{\frac{\ip{\blx_i}{\blx_j}}{\sqrt{\tau_2}}} \mathbf{1}\braces*{i \neq j}.
\]
Accounting for differences in scaling between the Gegenbauer polynomials considered in the previous section, a natural conjecture (in the isotropic/well-conditioned case) would be that $d^{-\ell/2}\E \norm*{\blDelta^{(\ell)}} \to 0$ as $n, d \to \infty$ provided that $\ell > q$, which would prove the desired polynomial approximation barrier of degree-$\floor{q}$ (that matches the spherical/hypercubic cases).
Showing this would be extremely challenging via the standard moment method.
This is because the combinatorial ``skeletonization'' process of~\cite{ghorbani2021linearized} involves repeatedly computing correlation-matrix entries, but the elegant identity of Equation~\eqref{eq:gegenbauer-correlation} no longer holds --- instead, only an \emph{approximate} form of this identity can be shown to hold (see our Lemma~\ref{lem:genredux}).
On the other hand, because the bound of Theorem~\ref{thm:gen-bound} only requires calculating the correlation matrix once, it is much simpler to work with.
In particular, we obtain the following corollary.

\begin{corollary}[Operator norm of Hermite matrices]\label{cor:Hnorm}
Let $\blx_1, \dots, \blx_n \simiid \scrN(\mathbf{0}, \blSigma)$. For any $\ell \geq 0$, the matrix $\blDelta^{(\ell)}$ satisfies
\[
\E \norm*{\blDelta^{(\ell)}} \lesssim_{\log} \sqrt{n} + n R^{-\ell/4}.
\]
Furthermore, if $\blSigma = \blI_d$ and $\ell \geq 4$ is even, we have the lower bound
\[
\E \norm*{\blDelta^{(\ell)}} \gtrsim nd^{\floor{\frac{\ell}{4}} - \frac{\ell}{2}}.
\]
\end{corollary}
Corollary~\ref{cor:Hnorm} directly implies a better approximation barrier of $\floor{\frac{4q}{3}}$, as shown in the next section.
We note that in the limit as the effective dimension $R \to \infty$ with fixed $n$, we expect entries of this matrix to be close to Hermite polynomials of independent Gaussian variables (by the CLT), and the operator norm to scale like $\sqrt{n}$ (as in Wigner-type ensembles); this behavior is captured by our bound. In the general polynomial scaling regime where both $n$ and $R$ are growing, this bound provides a novel approximation of the empirical kernel matrix, as we will see in the next section.

\begin{proof} We need to compute each of the terms that appear in the bound given by Theorem \ref{thm:gen-bound}, applied to the kernel $k(\blx,\bly) = \He_{\ell}\parens*{\frac{\ip{\blx}{\bly}}{\sqrt{\tau_2}}}$. 

\begin{itemize}
    \item Applying Lemma \ref{lem:hexp}, we have
    \begin{align*}
        n\sqrt{\E[\E_z[k(x_1,z)]^2]} &\lesssim n \sqrt{\E \abs*{\frac{\normt{\blSigma^{1/2} \blx_i}^2}{\tau_2} - 1}^{\ell}} \lesssim n\sqrt{R^{-\ell/2}} = nR^{-\ell/4},
    \end{align*}
    where the second-to-last inequality is a consequence of Whittle's inequality (Lemma \ref{lem:whittle}).


    
    \item  For the next term, we again apply Lemmas \ref{lem:hexp} and \ref{lem:whittle} to obtain
    \begin{align*}
        n\sqrt{\E_{\blx_1}(\E_{\blz} k(\blx_1, \blz))^2} &\lesssim n \sqrt{\E_{\blx_1} \abs*{\frac{\normt{\blSigma^{1/2} \blx_1}^2}{\tau_2} - 1}^{\ell}}\lesssim nR^{-\ell/4}.
    \end{align*}

    \item We bound the expected operator norm of the off-diagonal part of $\blG$ by the expected Frobenius norm. In particular, by Jensen's inequality, we have
    \[
    \E\norm{\diag^{\perp} \blG} \leq \sqrt{\sum_{i \neq j}\E(\E_{\blz}k(\blx_i, \blz)k(\blz, \blx_j))^2} \leq n \sqrt{\E_{\blx_1, \blx_2}(\E_{\blz} k(\blx_1, \blz)k(\blz, \blx_2))^2}.
    \]
    Using Lemmas \ref{lem:genredux} and \ref{lem:whittle}, along with the Cauchy-Schwarz inequality, we have, for some constants $c_{m,\ell}$ (depending only on $m$ and $\ell$),
    \begin{align*}
       &\E_{\blx_1, \blx_2}(\E_{\blz} k(\blx_1, \blz)k(\blz, \blx_2))^2 \\
       &= \E\parens*{\sum_{m=0}^{\floor{\ell/2}}c_{m,\ell} \parens*{\frac{\normt{\blSigma^{1/2}\blx_i}^2}{\tau_2} - 1}^{m}\parens*{\frac{\normt{\blSigma^{1/2}\blx_{i'}}^2}{\tau_2} - 1}^{m} \parens*{\blx_i^\top \blSigma \blx_{i'}}^{\ell-2m} \tau_2^{2m-\ell}}^2\\
      &\lesssim \sum_{m=0}^{\floor{\ell/2}} \E\parens*{\frac{\normt{\blSigma^{1/2}\blx_i}^2}{\tau_2} - 1}^{2m}\parens*{\frac{\normt{\blSigma^{1/2}\blx_{i'}}^2}{\tau_2} - 1}^{2m} \parens*{\blx_i^\top \blSigma \blx_{i'}}^{2\ell-4m} \tau_2^{4m-2\ell}\\
    &\lesssim \sum_{m=0}^{\floor{\ell/2}} R^{-2m} \tau_4^{\ell - 2m}\tau_2^{4m-2\ell}\\
    &\lesssim R^{-\ell}.
    \end{align*}

    For the diagonal part of $\blG$, we have
    \begin{align*}
        \E \norm{\diag \blG} &= \E \max_{i} \sum_{m=0}^{\floor{\ell/2}}c_{m,\ell} \parens*{\frac{\normt{\blSigma^{1/2}\blx_i}^2}{\tau_2} - 1}^{2m} \parens*{\blx_i^\top \blSigma \blx_{i}}^{\ell-2m} \tau_2^{2m-\ell}\\
        &\lesssim \sum_{m=0}^{\floor{\ell/2}} \tau_2^{2m - \ell} \E \max_i \parens*{\frac{\normt{\blSigma^{1/2}\blx_i}^2}{\tau_2} - 1}^{2m} \parens*{\blx_i^\top \blSigma \blx_{i}}^{\ell-2m} \\
        &\overset{(1)}{\lesssim}_{\log} \sum_{m=0}^{\floor{\ell/2}} \tau_2^{2m-\ell} \tau_4^{m} \tau_2^{-2m} \tau_2^{\ell-2m}\\
        &\lesssim 1,
    \end{align*}
    where inequality (1) uses Lemma \ref{lem:maxpoly} and bounds on moments of Gaussian quadratic forms \cite{magnus1978moments}. Combining the above bounds on the norms of the off-diagonal and diagonal parts of $\blG$, we can conclude that $\E\norm{\blG} \lesssim_{\log} nR^{-\ell/2}$.

    \item Lastly, we consider
    \begin{align*}
         \E \max_{1\leq i \leq n} \parens{k(\blz, \blx_i) - \E_{\blx_i} k(\blz, \blx_i)}^2 &= \E \max_i \brackets*{\He_\ell\parens*{\frac{\ip{\blx_i}{\blz}}{\sqrt{\tau_2}}} - \E_{\blx_i} \He_\ell\parens*{\frac{\ip{\blx_i}{\blz}}{\sqrt{\tau_2}}}}^2\\
        &\lesssim \E_{\blz} \log^c(n) \cdot \norm*{P_{\ell}(\cdot)}_{L^2}^2,\\
    \end{align*}
where we use Lemma \ref{lem:maxpoly} and define $P_{\ell}(\cdot) = \He_\ell\parens*{\frac{\ip{\cdot}{\blSigma^{1/2}\blz}}{\sqrt{\tau_2}}} - \E_{\blx} \He_\ell\parens*{\frac{\ip{\blx}{\blz}}{\sqrt{\tau_2}}}$, which is a standard Gaussian polynomial conditional on $\blz$. 
We can bound the $L^2$ norm using Lemma \ref{lem:genredux} by noting that
\begin{align*}
    \norm*{P_{\ell}}_{L^2}^2 &= \text{Var}\parens*{\He_{\ell}\parens*{\frac{\ip{\blx_i}{\blz}}{\sqrt{\tau_2}}} \Big{|} \blz}\\
    &\leq \E_{\blx_i} \brackets*{\He_{\ell}\parens*{\frac{\ip{\blx_i}{\blz}}{\sqrt{\tau_2}}}^2}\\
    &\leq \sum_{m=0}^{\floor{\ell/2}} c_{m,\ell} \parens*{\frac{\normt{\blSigma^{1/2}\blz}^2}{\tau_2} - 1}^{2m} \parens*{\blz^\top \blSigma \blz}^{\ell-2m} \tau_2^{2m-\ell}.\\
\end{align*}
Substituting this into the bound above and taking the expectation with respect to $\blz$ using the Cauchy-Schwarz inequality, Lemma \ref{lem:whittle}, and bounds on moments of quadratic forms \cite{magnus1978moments}, we obtain
\begin{align*}
    \E \max_{1\leq i \leq n} \parens{k(\blz, \blx_i) - \E_{\blz} k(\blz, \blx_i)}^2 &\lesssim \log^c(n) \sum_{m=0}^{\floor{\ell/2}} \E\brackets*{\parens*{\frac{\normt{\blSigma^{1/2}\blz}^2}{\tau_2} - 1}^{2m} \parens*{\blz^\top \blSigma \blz}^{\ell-2m}} \tau_2^{2m-\ell}\\
    &\lesssim \log^c(n) \sum_{m=0}^{\floor{\ell/2}} R^{-m} \tau_2^{\ell-2m}\tau_2^{2m-\ell}\\
    &\lesssim \log^c(n).
    \end{align*}
\end{itemize}
Combining these bounds in each of the terms of Theorem~\ref{thm:gen-bound} yields the desired upper bound. For the lower bound in the case where $\ell$ is even and $\blSigma =\blI_d$, we use Jensen's inequality to obtain
\begin{align*}
\E \norm*{\blDelta^{(\ell)}} \geq \norm*{\E \blDelta^{(\ell)}} = n \abs*{\E \He_\ell\parens*{\frac{\ip{\blx_1}{\blx_2}}{\sqrt{d}}}} \asymp n \abs*{\E \parens*{\frac{\normt{\blx_1}^2}{d} - 1}^{\ell/2}} \asymp n d^{-\ell/2} d^{\floor{\ell/4}},
\end{align*}
where the second line uses Lemma \ref{lem:hexp} and the last line uses the fact that the $k$-th central moment of a $\chi^2(d)$ random variable scales like $d^{\floor{k/2}}$ for a positive integer $k\geq 2$ (see, e.g., Lemma G.1 in \cite{oda2020fast}).
\end{proof}

\section{Kernel matrix approximation in the polynomial regime}\label{sec:generalK}
In this section, we show how Corollary~\ref{cor:Hnorm} can be used to provide new results for approximating general inner product kernel matrices with anisotropic Gaussian data. Recall that we consider $\blx_1, \dots, \blx_n \simiid \scrN(\mathbf{0}, \blSigma)$ and denote $\tau_k = \tr(\blSigma^k)$. We will consider the high-dimensional polynomial scaling regime where $c\leq \frac{n}{\tau_1^q} \leq C$, for some $q>0$ and constants $c, C > 0$. Let $k \colon \R^d \times \R^d \to \R$ be an inner product kernel of the form
\begin{align}\label{eq:kdef}
k(\blx, \blx') = f\parens*{\frac{\ip{\blx}{\blx'}}{\tr{\blSigma}}},
\end{align}
where $f \colon \R \to \R$ is assumed to be a $C^{\floor{2q} + 1 }$ function in a neighborhood of $0$ and is $L$-Lipschitz in a neighborhood $[1-\delta, 1+\delta]$, for some $\delta>0$.

We are interested in studying the behavior of the empirical kernel matrix $\blK \in \R^{n\times n}$, where $\blK_{ij} = k(\blx_i, \blx_j)$. In this section, we assume without loss of generality that $\norm{\blSigma} = 1$.
For conciseness, we will occasionally write $\tau \coloneqq \tau_1$. 

Next, we define the matrix 
\[
\bar{\blK} \coloneqq \sum_{\ell=0}^{\floor{\frac{4q}{3}}}\frac{f^{(\ell)}(0)}{\ell! \tau_1^\ell} (\blX \blX^\top)^{\odot \ell} + \sum_{\ell=\floor{\frac{4q}{3}} + 1}^{\floor{2q}}\frac{f^{(\ell)}(0)}{\ell! \tau_1^\ell}\tau_2^{\ell/2}\sum_{k=0}^{\floor{4q/3}} c_{k \ell} \blH^{(k)} + \parens*{f(1) - \sum_{j = 0}^{ \floor{\frac{4q}{3}}}\frac{f^{(j)}(0)}{j!}}\blI,
\]
where $c_{k \ell} = \frac{1}{k!} \E_{z \sim \scrN(0,1)} [z^\ell \He_k(z)]$ and
\[
H^{(k)}_{ij} \coloneqq \He_{k}\parens*{\frac{\ip{\blx_i}{\blx_j}}{\sqrt{\tau_2}}}.
\]

Note that $\bar{\blK}$ is the sum of a multiple of the identity matrix and a polynomial kernel matrix of degree at most $\floor{\frac{4q}{3}}$. Our main result in this section is the following theorem, proven in Appendix \ref{app:gaussianKproof}.

\begin{theorem}\label{thm:gaussianK}
   In the setting described above, 
   \[
   \norm{\blK - \bar{\blK}} \lesssim_{\log} \tau^{q-\frac{\floor{2q}}{2}-\frac{1}{2}} + \tau^{q - \frac{3}{4}\floor{\frac{4q}{3}} - \frac{3}{4}} + \tau^{-\frac{1}{2}}
   \]
   with probability at least $1 - \frac{c}{\log{n}}$. Hence, 
   $\norm{\blK - \bar{\blK}} \to 0$ in probability as $n, \tau \to \infty$.
\end{theorem}
To our knowledge, this is the sharpest known kernel approximation result with anisotropic Gaussian data in the general polynomial scaling regime. We recover the bounds developed for the linear \cite{el2010spectrum} and quadratic scaling regimes \cite{pandit2024universality}, while tightening the approximation result in the general polynomial scaling regime from a degree $\floor{2q}$ polynomial to a degree $\floor{4q/3}$ polynomial.  Due to our use of the sharp bounds from Corollary \ref{cor:Hnorm}, we are able to obtain a faster convergence guarantee than \cite{pandit2024universality} in the quadratic case (on the order of $d^{-1/2}$ instead of $d^{-1/12}$).  It is an interesting open question to study whether Theorem \ref{thm:gaussianK} can be improved further to the conjectured degree-$\floor{q}$ polynomial approximation, to match known results for the uniform spherical and hypercubic distributions. However, the lower bound we prove in Corollary \ref{cor:Hnorm} implies that a decomposition of $f$ in terms of univariate Hermite polynomials will not lead to the conjectured approximation result.  

In some sense, this tells us that the univariate Hermite basis is the wrong orthogonal decomposition to use to approximate general inner product kernels with Gaussian data, even when the covariance is isotropic!
In the special case of isotropy, we can form an alternative approximation by leveraging the sharp result for uniform data on the sphere (cf. Corollary \ref{cor:Gnorm}). Formally, the polar decomposition of a vector $\blx \sim \scrN(\mathbf{0}, \blI_d)$ into independent norm and unit vector terms allows us to approximate the kernel matrix by a degree-$\floor{q}$ polynomial of \textit{unit vectors} --- with random norm-based coefficients --- via a Gegenbauer polynomial expansion. 

\begin{proposition}\label{prop:Hnorm-isotropic}
    Let $\blx_1, \dots, \blx_n \simiid \scrN(\mathbf{0}, \blI_d)$. For each $i$, denote $r_i \coloneqq \normt{\blx_i}$ and $\blu_i \coloneqq \frac{\blx_i}{\normt{\blx_i}}$. Then, under the same assumptions as in Theorem \ref{thm:gaussianK}, we have
    \begin{align*}
        \norm{\blK - \bar{\blK}} \to 0
    \end{align*}
    in probability as $n,d \to \infty$, where we define the approximating matrix as 
    \[
    \bar{\blK} \coloneqq \sum_{\ell=0}^{\floor{2q}}\frac{f^{(\ell)}(0)}{\ell!d^{\frac{\ell}{2}}} \parens*{\frac{\blr \blr^\top}{d}}^{\odot \ell} \odot \parens*{\sum_{j=0}^{\floor{q}}c_{j\ell}^{(d)} \blQ^{(j)}} + \parens*{f(1) - \sum_{j=0}^{\floor{q}} \frac{f^{(j)}(0)}{j!}}\blI_n,
    \]
    \[
    c_{j\ell}^{(d)} \coloneqq B(d,j)\E_{\blx \sim \scrS^{d-1}(\sqrt{d})}[\ip{\blx}{\ble_1}^\ell Q_j^{(d)}(\sqrt{d}\ip{\blx}{\ble_1})].
    \]
    Above, $\blr$ is the vector with entries $r_i$, and $\blQ^{(j)}$ is the matrix with entries $Q_j^{(d)}(\ip{\sqrt{d}\blu_i}{\sqrt{d} \blu_j})$.
\end{proposition}
While Proposition~\ref{prop:Hnorm-isotropic} is relatively simple to prove given the tools we developed in the previous section (see Appendix \ref{app:isotropicK}), we have not seen it stated in this general form in the literature.
An intriguing question is whether this polar decomposition trick can be leveraged for anisotropic Gaussian data of the form $\x \sim \scrN(\mathbf{0},\blSigma)$.
Indeed, it is possible to decompose $\x = r \blv$ where $r := \norm{\blSigma^{-1/2} \x}$ and $\blv = \blSigma^{1/2} \blu$ (where $\blu$ is uniformly distributed on the sphere) and $r$ is independent of $\blv$.
However, identifying the correct basis for inner products of the form $\langle \blv_i, \blv_j \rangle$ is challenging since it will likely need to involve calculations with generalized ellipsoidal harmonics, and we do not believe an elegant identity of the form of Equation~\eqref{eq:gegenbauer-correlation} holds in general.
We leave this as an important direction for future work.

\subsection{Lower bound on the bias of KRR}\label{sec:biaslowerbound}
In Theorem \ref{thm:gaussianK}, we showed that general inner product kernel matrices with anisotropic Gaussian data are well-approximated by polynomial kernel matrices of degree $\floor{4q/3}$ under the high-dimensional scaling $n \asymp \tau^q$. In this section, we use similar Hermite decompositions to show that the generalization error of kernel ridge regression (KRR) is lower bounded by the bias of the best degree-$\floor{4q/3}$ polynomial approximation to the target function. The overall strategy we use is similar to the generalization analysis in \cite{pandit2024universality}, but we will consider the target function $g^*$ belonging to a family of generalized additive models. 
Specifically, we consider i.i.d. samples drawn from the model
\begin{align*}
y_i = g^*(\blx_i) + \epsilon_i,
\end{align*}
where $\blx_1,\ldots,\blx_n \simiid \scrN(\mathbf{0}, \blSigma)$, and $\epsilon_1,\ldots,\epsilon_n \simiid \scrN(0,\sigma^2)$. We will consider the target $g^*$ to take the form
\begin{align}\label{eq:gstar}
g^*(\blx) = \sum_{k=0}^K c_k g_k(\ip{\blx}{\blSigma^{-1/2}\blv_k}),
\end{align}
for some constant $K$, fixed unit vectors $\blv_k \in \scrS^{d-1}$, and univariate functions $g_k \in L^2(\scrN(0,1))$. Intuitively, this condition requires that $g^*$ depends only on $K$ scalar projections of the whitened input. 

Given these $n$ samples, the standard KRR estimator is constructed as
\[
\hat{g}(\blx) = \bly^\top(\blK + \lambda \blI_n)^{-1}\blk_{\blX}(\blx),
\]
where $\lambda \geq 0$ is the ridge regularization parameter, $\blk_{\blX}(\blx) \in \R^n$ is the vector with entries $k(\blx_i, \blx)$, and $\bly \in \R^n$ is the vector with entries $y_i$.

The main result of this subsection is stated below and proven in Appendix \ref{app:biasproof}.

\begin{theorem}[Lower bound on the bias of KRR]\label{thm:biaslowerbound}
    Consider the kernel regression estimate corresponding to the inner product kernel in Equation \eqref{eq:kdef} with ridge regularization parameter $\lambda \geq 0$. Assume that $g^*$ is of the form in Equation \eqref{eq:gstar} and that there exists some integer $L>4q-2$ such that $f^{(L+1)}$ is uniformly bounded by a constant. Then, in the scaling regime $n \asymp \tau^q$,
    \[
    \mathsf{Bias}(\hat{g}, g^*) \geq \inf_{p \in \scrP_{\leq \floor{\frac{4q}{3}}}} \norm{p - g^*}_{L^2}^2  - o_\tau(1),
    \]
    where $\scrP_{\leq k}$ is the space of multivariate polynomials of degree less than or equal to $k$ in $d$ dimensions.
\end{theorem}
Our theorem shows that even for this relatively simple class of target functions, KRR suffers from a polynomial approximation barrier and is unable to learn functions of degree greater than $\floor{4q/3}$. It is an interesting direction for future work to strengthen this bound to the conjectured $\floor{q}$ lower bound (which would match the uniform spherical and binary hypercube case, as analyzed in \cite{ghorbani2021linearized, mei2022generalization2}).
Based on the bias and variance calculations in~\cite{ghorbani2021linearized}, we believe that an optimal degree-$\floor{q}$ approximation result would be instrumental for strengthening the bias lower bound (as well as providing a matching upper bound and characterizing the variance of KRR).

\section*{Acknowledgements}
This work was supported in part by the NSF AI Institute AI4OPT, NSF 2112533. VM was supported by the NSF (through award CCF-2239151 and award IIS-2212182), an Adobe Data Science Research Award, and an Amazon Research Award.

\bibliography{refs}

\appendix
\section{Auxiliary Lemmas}
\begin{lemma}[Special case of Whittle's inequality \cite{whittle1960bounds}]\label{lem:whittle}
Let $\blz \in \R^d$ be a standard normal vector, and let $\blSigma$ be a diagonal matrix. Then, for $s \geq 2$,
\[
\E\brackets*{\abs*{\frac{\blz^\top \blSigma \blz}{\tr{\blSigma}} - 1}^s} \leq C(s) (\tr (\blSigma^2))^{s/2} \tr(\blSigma)^{-s}.
\]
\end{lemma}


\begin{lemma}[Hermite multiplication formula, e.g., Theorem 2.4 in \cite{davis2024general}]\label{lem:mult}
The Hermite expansion of $\He_\ell(\gamma x)$ is given by
\[
\He_\ell(\gamma x) = \ell! \sum_{k=0}^{\floor{\ell/2}} \frac{1}{2^k k! (\ell-2k)!} \gamma^{\ell-2k} (\gamma^2-1)^k \He_{\ell-2k}(x).
\]
\end{lemma}

\begin{lemma}[Lemma D.2 in \cite{nguyen2020global}]\label{lem:unitredux}
Let $\blx, \bly$ be unit vectors in $\R^d$ and $\blz \sim \scrN(\mathbf{0}, \blI_d)$. Then, 
\[
\E\brackets*{\He_k(\ip{\blx}{\blz})\He_\ell(\ip{\bly}{\blz})} = \delta_{kl} \ell! \ip{\blx}{\bly}^\ell.
\]
\end{lemma}

\begin{lemma}[Conditional expectation of Hermite matrix entries]\label{lem:hexp} Let $\blx_2 \sim \scrN(\mathbf{0}, \blSigma)$. Then, for fixed $\blx_1$,
\begin{equation}
\begin{aligned}
\E_{\blx_2}\brackets*{\He_\ell\parens*{\frac{\ip{\blx_1}{\blx_2}}{\sqrt{\tau_2}}}}
&= \mathbf{1}\braces*{\ell \text{ is even}}\frac{\ell!}{2^{\ell/2}(\ell/2)!}\parens*{\frac{\normt{\blSigma^{1/2} \blx_1}^2}{\tau_2} - 1}^{\ell/2}.
\end{aligned}
\end{equation}
\end{lemma}
\begin{proof}
The desired expectation can be computed directly as
\begin{align*}
\E_{\blx_2}\brackets*{\He_\ell\parens*{\frac{\ip{\blx_1}{\blx_2}}{\sqrt{\tau_2}}}} &= \E_{z \sim \scrN(0,1)}\brackets*{\He_\ell\parens*{\frac{\normt{\blSigma^{1/2}\blx_1}}{\sqrt{\tau_2}}z}}\\
&= \ell! \sum_{k=0}^{\floor{\ell/2}} \frac{1}{2^k k! (\ell-2k)!} \parens*{\frac{\normt{\blSigma^{1/2}\blx_1}}{\sqrt{\tau_2}}}^{\ell-2k} \parens*{\frac{\normt{\blSigma^{1/2}\blx_1}^2}{\tau_2}-1}^k \E_{z \sim \scrN(0,1)} \He_{\ell-2k}(z),\\
\end{align*}
where the last line follows from Lemma \ref{lem:mult}. The result follows immediately by noting that $\E_{z \sim \scrN(0,1)} \He_{\ell-2k}(z) = \mathbf{1}\braces{k = \ell/2}$. 
\end{proof}

\begin{lemma}[Conditional correlation of Hermite matrix entries]\label{lem:genredux} Let $\blx_2 \sim \scrN(\mathbf{0}, \blSigma)$. Then, for fixed $\blx_1$ and $\blx_3$, and any two indices $\ell \leq \ell'$,
\begin{equation}
\begin{aligned}
&\E_{\blx_2}\brackets*{\He_\ell\parens*{\frac{\ip{\blx_1}{\blx_2}}{\sqrt{\tau_2}}}\He_{\ell'}\parens*{\frac{\ip{\blx_2}{\blx_3}}{\sqrt{\tau_2}}}} \\
&=  \sum_{j=0}^{\floor{\ell/2}}\frac{\ell! (\ell')!}{2^{2j + \frac{\ell' - \ell}{2}}j! (\frac{\ell' - \ell}{2} + j)! (\ell-2j)!} \parens*{\frac{\normt{\blSigma^{1/2}\blx_1}^2}{\tau_2} - 1}^{j}\parens*{\frac{\normt{\blSigma^{1/2}\blx_3}^2}{\tau_2} - 1}^{\frac{\ell' - \ell}{2} + j} \parens*{\blx_1^\top \blSigma \blx_3}^{\ell-2j} \tau_2^{2j-\ell}.
\end{aligned}
\end{equation}
\end{lemma}
\begin{proof}[Proof of Lemma~\ref{lem:genredux}]
The desired expectation is
\begin{align*}
&\E_{\blx_2}\brackets*{\He_\ell\parens*{\frac{\ip{\blx_1}{\blx_2}}{\sqrt{\tau}}}\He_{\ell'}\parens*{\frac{\ip{\blx_2}{\blx_3}}{\sqrt{\tau}}}}\\
&= \E_{\blz_2}\brackets*{\He_\ell\parens*{\frac{\normt{\blSigma^{1/2}\blx_1}}{\sqrt{\tau_2}}\ip*{\frac{\blSigma^{1/2}\blx_1}{\normt{\blSigma^{1/2}\blx_1}}}{\blz_2}}\He_{\ell'}\parens*{\frac{\normt{\blSigma^{1/2}\blx_3}}{\sqrt{\tau_2}}\ip*{\frac{\blSigma^{1/2}\blx_3}{\normt{\blSigma^{1/2}\blx_3}}}{\blz_2}}},
\end{align*}
where $\blz_2 = \blSigma^{-1/2} \blx_2$ has standard normal distribution. We can now proceed by expanding each Hermite polynomial using the Hermite multiplication theorem (Lemma \ref{lem:mult}) to get 
\begin{align*}
    \sum_{j=0}^{\floor{\ell/2}} \sum_{k=0}^{\floor{\ell'/2}} \frac{\ell! \cdot (\ell')!}{2^{j+k} j! k! (\ell-2j)!(\ell'-2k)!}&\parens*{\frac{\normt{\blSigma^{1/2}\blx_1}}{\sqrt{\tau_2}}}^{\ell-2j} \parens*{\frac{\normt{\blSigma^{1/2}\blx_3}}{\sqrt{\tau_2}}}^{\ell'-2k}\\
    &\cdot \parens*{\frac{\normt{\blSigma^{1/2}\blx_1}^2}{\tau_2} - 1}^{j}\parens*{\frac{\normt{\blSigma^{1/2}\blx_3}^2}{\tau_2} - 1}^{k}\\
    &\cdot \E_{\blz_2}\brackets*{\He_{\ell-2j}\parens*{\ip*{\frac{\blSigma^{1/2}\blx_1}{\normt{\blSigma^{1/2}\blx_1}}}{\blz_2}}\He_{\ell'-2k}\parens*{\ip*{\frac{\blSigma^{1/2}\blx_3}{\normt{\blSigma^{1/2}\blx_3}}}{\blz_2}}}.\\
\end{align*}
Recall that we consider, without loss of generality, the case where $\ell \leq \ell'$.
Next, we apply Lemma \ref{lem:unitredux} to evaluate the expectations above, yielding
\begin{align*}
    &\sum_{j=0}^{\floor{\ell/2}}\frac{\ell! \cdot (\ell')!}{2^{2j + \frac{\ell' - \ell}{2}} j! (\frac{\ell' - \ell}{2} + j)! (\ell-2j)!} \parens*{\frac{\normt{\blSigma^{1/2}\blx_1}}{\sqrt{\tau_2}}}^{\ell-2j}\parens*{\frac{\normt{\blSigma^{1/2}\blx_3}}{\sqrt{\tau_2}}}^{\ell-2j}\\
    &\cdot \parens*{\frac{\normt{\blSigma^{1/2}\blx_1}^2}{\tau_2} - 1}^{j}\parens*{\frac{\normt{\blSigma^{1/2}\blx_1}^2}{\tau_2} - 1}^{\frac{\ell' - \ell}{2} + j} \parens*{\frac{\blx_1^\top \blSigma \blx_3}{\normt{\blSigma^{1/2}{\blx_3}}\normt{\blSigma^{1/2}{\blx_3}}}}^{\ell-2j}\\
    &= \sum_{j=0}^{\floor{\ell/2}}\frac{\ell! \cdot (\ell')!}{2^{2j + \frac{\ell' - \ell}{2}} j! (\frac{\ell' - \ell}{2} + j)! (\ell-2j)!} \parens*{\frac{\normt{\blSigma^{1/2}\blx_1}^2}{\tau_2} - 1}^{j}\parens*{\frac{\normt{\blSigma^{1/2}\blx_3}^2}{\tau_2} - 1}^{\frac{\ell' - \ell}{2} + j} \parens*{\blx_1^\top \blSigma \blx_3}^{\ell-2j} \tau_2^{2j-\ell}.
\end{align*}
\end{proof}

\begin{lemma}[Expected maximum of polynomials under hypercontractivity]\label{lem:maxpoly}
Let $\scrP$ be a probability measure satisfying the following hypercontractivity property: 
\[
\norm{Q}_{L_q} \leq (q-1)^{k/2} \norm{Q}_{L_2}, 
\]
for any polynomial $Q$ of degree at most $k$ in $d$ variables and any integer $q \geq 2$. In particular, this property is satisfied for the standard normal distribution in $\R^d$, the uniform distribution on $\scrS^{d-1}$, and the uniform distribution on the $d$-dimensional binary hypercube.

Let $\blz_1, \dots, \blz_n$ be i.i.d. random variables from $\scrP$ and let $Q$ be a polynomial of degree $k \geq 1$. Then, for $s \geq \frac{2}{k}$,
\[
\E \max_{1\leq i \leq n} |Q(\blz_i)|^s \lesssim (\log{n})^{ks/2} \norm*{Q}_{L^2}^s,
\]
where $\norm*{Q}_{L^2} \coloneqq (\E [Q(z)^2])^{1/2}$ denotes the $L^2$ norm of $Q$ with respect to the distribution $\scrP$, and the suppressed universal constant depends only on $k$ and $s$. 
\end{lemma}
\begin{proof}
We first obtain a tail bound for the maximum, following the approach in the proof of Proposition 5.48 in \cite{aubrun2017alice}. Let $q \geq 2$ be a constant we will fix later in the proof. By a union bound and Markov's inequality,
\begin{align*}
\P\braces*{\max_{1\leq i \leq n} |Q(\blz_i)|^s  > t\norm*{Q}_{L^2}^s} &\leq n \P\braces*{|Q(\blz_1)|^s  > t\norm*{Q}_{L^2}^s}\\
&\leq n t^{-q} \norm*{Q}_{L^2}^{-qs} \E\brackets*{|Q(\blz_1)|^{qs}}\\
&\overset{(1)}{\leq} n t^{-q} (qs)^{kqs/2}
\end{align*}
where inequality (1) follows from hypercontractivity. Choosing $q = \frac{t^{2/ks}}{es}$, which satisfies $q\geq 2$ for $t \geq (2es)^{ks/2}$, we obtain
\[
\P\braces*{\max_{1\leq i \leq n} |Q(\blz_i)|^s  > t\norm*{Q}_{L^2}^s} \leq n \text{exp}\parens*{-\frac{k}{2e}t^{\frac{2}{ks}}}.\\
\]
Letting $C>0$ be a constant and integrating the tail bound, we obtain
\begin{align*}
    \E \max_{1\leq i \leq n} |Q(\blz_i)|^s &\leq \norm{Q}_{L^2}^s \brackets*{ (C\log{n})^{ks/2} + \int_{ (C\log{n})^{ks/2}}^\infty n \text{exp}\parens*{-\frac{k}{2e}t^{\frac{2}{ks}}} dt} \\
    &\lesssim \norm{Q}_{L^2}^s \brackets*{ (C\log{n})^{ks/2} + n \int_{C'\log{n}}^\infty \text{exp}\parens*{-u} u^{\frac{ks}{2}-1}du}\\
    &= \norm{Q}_{L^2}^s \brackets*{ (C\log{n})^{ks/2} + n \Gamma\parens*{\frac{ks}{2}, C'\log{n}}}\\
    &\leq \norm{Q}_{L^2}^s \brackets*{ (C\log{n})^{ks/2} + n \exp\parens*{-C' \log{n}} (C'\log{n})^{\frac{ks}{2}-1}},\\
\end{align*}
where the second line uses the substitution $u = \frac{k}{2e}t^{2/ks}$, the third line uses the definition of the incomplete Gamma function, and the last line uses the upper bound $\Gamma(a,x) \leq ae^{-x}x^{a-1}$, which holds for $a\geq 1$ and $x > a$ \cite[Proposition 4.4.3]{gabcke2015neue}. Noting that we can choose $C$ so that $C'=1$ completes the proof.

\end{proof}

\begin{lemma}[Operator norm of Hadamard product with outer product]\label{lem:hadamard-outer}
Let $\blP \in \R^{n \times n}$ and $\bla \in \R^n$. Then,
\[
\norm{\bla \bla^\top \odot \blP} \leq \norm{\bla}_\infty^2 \norm{\blP}.
\]
\begin{proof}
We directly have
\begin{align*}
    \norm{\bla \bla^\top \odot \blP} &= \max_{\normt{\blu} = 1} \normt{(\bla \bla^\top \odot \blP) \blu}\\
    &= \max_{\normt{\blu} = 1} \sqrt{\sum_{i=1}^n |\sum_{j=1}^n P_{ij} a_i a_j u_j|^2}\\
    &\leq \norm{\bla}_\infty \max_{\normt{\blu} = 1} \sqrt{\sum_{i=1}^n |\sum_{j=1}^n P_{ij} a_j u_j|^2}\\
    &= \norm{\bla}_\infty \norm{\blP (\bla \odot \blu)}\\
    &\leq \norm{\bla}_\infty \norm{\blP} \normt{\bla \odot \blu}\\
    &= \norm{\bla}_\infty \norm{\blP} \sqrt{\sum_{i=1}^n |a_i u_i|^2}\\
    &\leq \norm{\bla}_{\infty}^2 \cdot \norm{\blP}.
\end{align*}
\end{proof}
\end{lemma}

\section{Proof of Theorem \ref{thm:gaussianK}}\label{app:gaussianKproof}
Before continuing with the proof, we define the following ``good event'': 
\begin{align}\label{eq:good-event}
\scrE \coloneqq \braces*{\blX \colon \max_{1\leq i, j \leq n} \abs*{\frac{\ip{\blx_i}{\blx_j}}{\tau_1} - \delta_{ij}} \lesssim \tau_1^{-1} \tau_2^{1/2} \log{n} \lesssim \tau_1^{-1/2} \log{n}}
\end{align}

Applying standard concentration inequalities for polynomials of Gaussian random variables (e.g., \cite[Corollary 5.49]{aubrun2017alice}) and a union bound, we can obtain that $\P\brackets{\scrE} \geq 1 - \frac{c}{n^2}$ for some constant $c>0$. Hence, we will condition on the event $\scrE$ (Equation~\eqref{eq:good-event}) for the remainder of the proof.

\subsection{Off-diagonal part}
For this part, first write the order $\floor{2q}+1$ Taylor expansion of the kernel around $0$. For any $i \neq j$, we have
\[
\blK_{ij} = \sum_{\ell=0}^{\floor{2q}}\frac{f^{(\ell)}(0)}{\ell! \tau^\ell}\ip{\blx_i}{\blx_j}^\ell + \frac{f^{(\floor{2q}+1)}(\zeta_{ij})}{\ell! \tau^{\floor{2q}+1}}\ip{\blx_i}{\blx_j}^{\floor{2q}+1},
\]
for some $\zeta_{ij}$ between $\frac{\ip{\blx_i}{\blx_j}}{\tau}$ and $0$. Hence, we can write 

\[
\diag^\perp(\blK) = \blR + \blS,
\]
where we define 
\begin{align*}
\blR_{ij} &\coloneqq\sum_{\ell=0}^{\floor{2q}}\frac{f^{(\ell)}(0)}{\ell! \tau^\ell}\ip{\blx_i}{\blx_j}^\ell\mathbf{1}\braces{i \neq j}\\
\blS_{ij} &\coloneqq \frac{f^{(\floor{2q}+1)}(\zeta_{ij})}{\ell! \tau^{\floor{2q}+1}}\ip{\blx_i}{\blx_j}^{\floor{2q}+1}\mathbf{1}\braces{i \neq j}.\\
\end{align*}
First, we bound the norm of $\blS$ as follows:
\begin{align*}
    \norm{\blS}^2 \leq \norm{\blS}_F^2 &\lesssim n^2 \max_{i\neq j} \abs*{\frac{f^{(\floor{2q}+1)}(\zeta_{ij})}{\ell! \tau^{\floor{2q}+1}}}^2 \abs*{\ip{\blx_i}{\blx_j}}^{2\floor{2q}+2}\\
    &\overset{(1)}{\lesssim} n^2 \max_{i\neq j} \abs*{\frac{f^{(\floor{2q}+1)}(\zeta_{ij})}{\tau^{\floor{2q}+1}}}^2 \tau^{2\floor{2q}+2} \tau^{-\floor{2q}-1} (\log{n})^{2\floor{2q}+2}\\
    &\overset{(2)}{\lesssim} n^2 \tau^{-\floor{2q} -1} (\log{n})^{2\floor{2q}+2}\\
    &\lesssim \tau^{2q - \floor{2q} - 1}(\log{n})^{2\floor{2q}+2}.\\
\end{align*}
So, we can conclude that $\norm{\blS} \lesssim \tau^{q-\frac{\floor{2q}}{2} - \frac{1}{2}}(\log{n})^{\floor{2q}+1}$. Inequality (1) above relies on the event $\scrE$ and the fact that $\tau_2 \lesssim \tau_1$ (because we assumed that $\norm{\blSigma} = 1$), and (2) additionally uses the fact that $f^{(\floor{2q}+1)}$ is continuous in a neighborhood of $0$, so $\max_{i\neq j} \abs*{f^{(\floor{2q}+1)}(\zeta_{ij})} \leq C$ for some $C>0$ as $\tau \to \infty$.

To understand the behavior of $\blR$, we first expand each monomial in terms of Hermite polynomials: For $i \neq j$, we have

\begin{align*}
    \blR_{ij} &= \sum_{\ell=0}^{\floor{2q}}\frac{f^{(\ell)}(0)}{\ell! }\tau_2^{\ell/2}\tau_1^{-\ell}\parens*{\frac{\ip{\blx_i}{\blx_j}}{\sqrt{\tau_2}}}^\ell\\
    &= \sum_{\ell=0}^{\floor{2q}}\frac{f^{(\ell)}(0)}{\ell! }\tau_2^{\ell/2}\tau_1^{-\ell}\sum_{k=0}^{\ell} c_{k \ell} \He_k\parens*{\frac{\ip{\blx_i}{\blx_j}}{\sqrt{\tau_2}}},
\end{align*}
where $c_{k \ell} = \frac{1}{k!} \E_{z \sim \scrN(0,1)} [z^\ell \He_k(z)]$. Next, define the matrix
\[
    \bar{\blR}_{ij} = \left(\sum_{\ell=0}^{\floor{2q}}\frac{f^{(\ell)}(0)}{\ell! }\tau_2^{\ell/2}\tau_1^{-\ell}\sum_{k=0}^{\min\{\ell, \floor{4q/3}\}} c_{k \ell} \He_k\parens*{\frac{\ip{\blx_i}{\blx_j}}{\sqrt{\tau_2}}}\mathbf{1}\right)\braces{i \neq j},
\]
Note that $\bar{\blR} = \diag^{\perp} \bar{\blK}$. We aim to show that $\Vert \blR - \bar{\blR}\Vert \to 0$ as $n \to \infty$. 
We have
\begin{align*}
    \E\norm{\blR - \bar{\blR}} &\lesssim \sum_{\ell = \floor{4q/3}+1}^{\floor{2q}} \tau_2^{\ell/2} \tau_1^{-\ell} \sum_{k= \floor{4q/3} + 1}^{\ell} \norm{\blDelta^{(k)}}\\
    &\overset{(1)}{\lesssim_{\log}} \sum_{\ell = \floor{4q/3}+1}^{\floor{2q}} \tau_2^{\ell/2} \tau_1^{-\ell} \sum_{k= \floor{4q/3} + 1}^{\ell} (\sqrt{n} + n\tau_2^{-k/2}\tau_4^{k/4})\\
    &\lesssim \sum_{\ell=\floor{4q/3}+1}^{\floor{2q}} \tau_1^{-\ell/2}\sqrt{n} + \sum_{\ell = \floor{4q/3} + 1}^{\floor{2q}} \sum_{k = \floor{4q/3} + 1}^{\ell} \tau_2^{\ell/2 - k/2}\tau_1^{-\ell}\tau_4^{k/4}n,
    \end{align*}
where inequality (1) substitutes Corollary~\ref{cor:Hnorm} and we used $\tau_2 \lesssim \tau_1$.
Next, we again use the fact that $\tau_2 \lesssim \tau_1$ and $\tau_4 \lesssim \tau_1$ to obtain
    \begin{align*}
    \E\norm{\blR - \bar{\blR}} &\lesssim \sum_{\ell=\floor{4q/3}+1}^{\floor{2q}} \tau_1^{-\ell/2}\sqrt{n} + \sum_{\ell = \floor{4q/3} + 1}^{\floor{2q}} \sum_{k = \floor{4q/3} + 1}^{\ell} \tau_1^{\ell/2 - k/2 - \ell + k/4}n\\
    &\lesssim \tau_1^{\frac{q-\floor{4q/3} - 1}{2}} + \sum_{\ell = \floor{4q/3} + 1}^{\floor{2q}} \sum_{k = \floor{4q/3} + 1}^{\ell} \tau_1^{q - \ell/2 - k/4}\\
    &\lesssim \tau_1^{\frac{q-\floor{4q/3} - 1}{2}} + \tau_1^{q - \frac{3}{4}\floor{\frac{4q}{3}} - \frac{3}{4}}\\
    &\lesssim \tau_1^{q - \frac{3}{4}\floor{\frac{4q}{3}} - \frac{3}{4}}
\end{align*}
So, by Markov's inequality, with probability at least $1- \frac{1}{\log{n}}$, we have 
\[
\norm{\blR - \bar{\blR}} \lesssim_{\log} \tau_1^{q - \frac{3}{4}\floor{\frac{4q}{3}} - \frac{3}{4}}.
\]
Combining the above, we can conclude that $\norm{\diag^{\perp}\blK - \diag^{\perp}\bar{\blK}} \to 0
$ with probability tending to 1 as $\tau \to \infty$.

\subsection{Diagonal part}
The diagonal part of the error is given by
\begin{align*}
    &\norm{\diag{\blK} - \diag{\bar{\blK}}} =\\ 
    &\max_{1\leq i \leq n} \abs*{f\parens*{\frac{\normt{\blx_i}^2}{\tau}} - \sum_{\ell=0}^{\floor{2q}}\frac{f^{(\ell)}(0)}{\ell! }\tau_2^{\ell/2}\tau_1^{-\ell}\sum_{k=0}^{\floor{4q/3}} c_{k \ell} \He_k\parens*{\frac{\normt{\blx_i}^2}{\sqrt{\tau_2}}} - f(1) + \sum_{j = 0}^{ \floor{\frac{4q}{3}}}\frac{f^{(j)}(0)}{j!}}\\
    &\leq \underbrace{\max_{1\leq i \leq n} \abs*{f\parens*{\frac{\normt{\blx_i}^2}{\tau}} - f(1)}}_{T_1} + \underbrace{\max_{1\leq i \leq n} \abs*{\sum_{\ell=0}^{\floor{2q}}\frac{f^{(\ell)}(0)}{\ell! }\tau_2^{\ell/2}\tau_1^{-\ell}\sum_{k=0}^{\floor{4q/3}} c_{k \ell} \He_k\parens*{\frac{\normt{\blx_i}^2}{\sqrt{\tau_2}}} -  \sum_{j = 0}^{ \floor{\frac{4q}{3}}}\frac{f^{(j)}(0)}{j!}}}_{T_2}.
\end{align*}
Here, by the Lipschitz assumption on $f$, $T_1$ is bounded for sufficiently large $n,\tau$ under the event $\scrE$ as
\[
T_1 \leq L \max_{1 \leq i \leq n} \abs*{\frac{\normt{\blx_i}^2}{\tau} - 1} \lesssim L \tau_1^{-1/2} \log{n}.
\]
For $T_2$, note that
\begin{align*}
    &\abs*{\sum_{\ell=0}^{\floor{2q}}\frac{f^{(\ell)}(0)}{\ell! }\tau_2^{\ell/2}\tau_1^{-\ell}\sum_{k=0}^{\floor{4q/3}} c_{k \ell} \He_k\parens*{\frac{\normt{\blx_i}^2}{\sqrt{\tau_2}}} -  \sum_{j = 0}^{ \floor{\frac{4q}{3}}}\frac{f^{(j)}(0)}{j!}}\\
    &\leq \abs*{\sum_{\ell = 0}^{\floor{4q/3}}\frac{f^{(\ell)}(0)}{\ell! }\parens*{\frac{\normt{\blx_i}^2}{\tau_1}}^\ell - \sum_{\ell = 0}^{\floor{4q/3}} \frac{f^{(\ell)}(0)}{\ell! }} + \abs*{\sum_{\ell = \floor{4q/3} + 1}^{\floor{2q}} \frac{f^{(\ell)}(0)}{\ell! }\tau_2^{\ell/2}\tau_1^{-\ell}\sum_{k=0}^{\floor{4q/3}} c_{k \ell} \He_k\parens*{\frac{\normt{\blx_i}^2}{\sqrt{\tau_2}}}}.    
\end{align*}
By the claim proven in Appendix E.3 of \cite{donhauser2021rotational}, the first term is bounded (for every $i$) under the event $\scrE$ as 
\[
\abs*{\sum_{\ell = 0}^{\floor{4q/3}}\frac{f^{(\ell)}(0)}{\ell! }\parens*{\frac{\normt{\blx_i}^2}{\tau_1}}^\ell - \sum_{\ell = 0}^{\floor{4q/3}} \frac{f^{(\ell)}(0)}{\ell! }} \lesssim_{\log} \tau_1^{-1/2}.
\]
So, we can bound $T_2$ on the event $\scrE$ as
\begin{align*}
    T_2 &\lesssim_{\log} \tau_1^{-1/2} + \max_{1 \leq i \leq n} \abs*{\sum_{\ell = \floor{4q/3} + 1}^{\floor{2q}} \frac{f^{(\ell)}(0)}{\ell! }\tau_2^{\ell/2}\tau_1^{-\ell}\sum_{k=0}^{\floor{4q/3}} c_{k \ell} \He_k\parens*{\frac{\normt{\blx_i}^2}{\sqrt{\tau_2}}}}\\
    &\lesssim \tau_1^{-1/2} +\sum_{\ell = \floor{4q/3} + 1}^{\floor{2q}} \sum_{k=0}^{\floor{4q/3}} \tau_2^{\ell/2}\tau_1^{-\ell} \max_{i}\abs*{\He_k\parens*{\frac{\normt{\blx_i}^2}{\sqrt{\tau_2}}}}\\
    &\lesssim_{\log} \tau_1^{-1/2} + \sum_{\ell = \floor{4q/3} + 1}^{\floor{2q}} \sum_{k=0}^{\floor{4q/3}} \tau_2^{\ell/2}\tau_1^{-\ell}\tau_1^k \tau_2^{-k/2}\\
    &\lesssim \tau_1^{-1/2} + \tau_1^{-1} \tau_2^{1/2} \\
    &\leq \tau_1^{-1/2}.
\end{align*}
Above, we used the fact that $\tau_2 \lesssim \tau_1$.
Combining the above, we can conclude, on the event $\scrE$, that
\[
\norm{\diag{\blK} - \diag{\bar{\blK}}} \lesssim_{\log} \tau_1^{-1/2}.
\]
\qed

\section{Proof of Proposition \ref{prop:Hnorm-isotropic}}\label{app:isotropicK}
Note that in the isotropic case we have $\tau_k = d$ for all $k$. Following the beginning of the proof of Theorem~\ref{thm:gaussianK} in Appendix \ref{app:gaussianKproof}, we separate the error into diagonal and off-diagonal components. For the off-diagonal component, we use the same Taylor decomposition to write
\[
\diag^\perp(\blK) = \blR + \blS,
\]
where
\begin{align*}
\blR_{ij} &\coloneqq \sum_{\ell=0}^{\floor{2q}}\frac{f^{(\ell)}(0)}{\ell! d^\ell}\ip{\blx_i}{\blx_j}^\ell\mathbf{1}\braces{i \neq j}\\
\blS_{ij} &\coloneqq \frac{f^{(\floor{2q}+1)}(\zeta_{ij})}{\ell! d^{\floor{2q}+1}}\ip{\blx_i}{\blx_j}^{\floor{2q}+1}\mathbf{1}\braces{i \neq j},\\
\end{align*}
and the same argument as in the proof of Theorem \ref{thm:gaussianK} shows that $\norm{\blS} = o_d(1)$. So it suffices to approximate the matrix $\blR$. Consider a single term in $\blR$:
\begin{align*}
   \blR^{(\ell)}_{ij} \coloneqq \frac{f^{(\ell)}(0)}{\ell! d^\ell}\ip{\blx_i}{\blx_j}^\ell &=  \frac{f^{(\ell)}(0)}{\ell!} \parens*{\frac{r_i}{\sqrt{d}}}^\ell \parens*{\frac{r_j}{\sqrt{d}}}^{\ell} \ip{\blu_i}{\blu_j}^\ell\\
   &=  \frac{f^{(\ell)}(0)}{\ell!} \parens*{\frac{r_i}{\sqrt{d}}}^\ell \parens*{\frac{r_j}{\sqrt{d}}}^{\ell} d^{-\ell/2} \parens*{\frac{\ip{\tilde{\blu_i}}{\tilde{\blu_j}}}{\sqrt{d}}}^\ell,
\end{align*}
where $r_i \coloneqq \normt{\blx_i}$, $\blu_i \coloneqq \frac{\blx_i}{r_i}$, and $\tilde{\blu_i} \coloneqq \sqrt{d} \blu_i$. Define the function 
\[
h(z) = \parens*{\frac{z}{\sqrt{d}}}^\ell.
\]
Let $\tau_{d-1}$ be the uniform distribution on $\sqrt{d} \cdot \scrS^{d-1}$ and $\tilde{\tau}_{d-1}$ be the distribution of $\sqrt{d}\ip{\blz}{\ble_1}$, when $\blz \sim \tau_{d-1}$. Then, observe that 
\begin{align*}
    \E_{z \sim \tilde{\tau}_{d-1}} [h(z)^2] &= \E_{\blz \sim \tau_{d-1}}\ip{\blz}{\ble_1}^{2\ell} \leq C,
\end{align*}
for some constant $C$ independent of $d$. Here, the last inequality follows from hypercontractivity of the spherical distribution. Hence $h \in L^2(\tilde{\tau}_{d-1})$, with norm independent of $d$. Recalling that the Gegenbauer polynomails $Q_k^{(d)}$ form an orthogonal basis for this space, we can expand $h$ as
\[
h(z) = \sum_{j=0}^\ell \alpha_j Q_j^{(d)}(z),
\]
where $\alpha_j = B(d,j)\E_{z \sim \tilde{\tau}_{d-1}}[h(z)Q_j^{(d)}(z)]$. Here, $B(d,j) \asymp d^j$ is the number of spherical harmonics of degree $j$ in $d$ dimensions. We can bound the coefficients using the Cauchy-Schwarz inequality and Equation \eqref{eq:gegenbauer-norm} as
\[
|\alpha_j| \leq  B(d,j) \cdot C \cdot \norm{Q_j^{(d)}}_{L^2(\tilde{\tau}_{d-1})} \lesssim \sqrt{B(d,j)} \asymp d^{j/2}.
\]
Using this decomposition, we can write each term of $\blR$ as
\[
\blR^{(\ell)}_{ij}  = \frac{f^{(\ell)}(0)}{\ell!} \parens*{\frac{r_i}{\sqrt{d}}}^\ell \parens*{\frac{r_j}{\sqrt{d}}}^{\ell} d^{-\ell/2} \sum_{j=0}^\ell \alpha_{j\ell} Q_j^{(d)}(\ip{\tilde{\blu_i}}{\tilde{\blu_j}}),
\]
where $|\alpha_{j\ell}| \lesssim d^{j/2}$. Next, define the matrix $\bar{\blR}^{(\ell)}$ with off-diagonal entries
\[
\bar{\blR}^{(\ell)}_{ij} \coloneqq \frac{f^{(\ell)}(0)}{\ell!} \parens*{\frac{r_i}{\sqrt{d}}}^\ell \parens*{\frac{r_j}{\sqrt{d}}}^{\ell} d^{-\ell/2} \sum_{j=0}^{\floor{q}} \alpha_{j\ell} Q_j^{(d)}(\ip{\tilde{\blu_i}}{\tilde{\blu_j}}).
\]
Note these two matrices only differ in the case $\ell > \floor{q}$. Then,  by the triangle inequality and recalling the definition of $\blDelta^{(j)}$ from Corollary \ref{cor:Gnorm}, we can write
\begin{align*}
    \E\norm{\blR - \bar{\blR}} &\leq \sum_{\ell = \floor{q} + 1}^{\floor{2q}}\E \norm{\blR^{(\ell)} - \bar{\blR}^{(\ell)}}\\
    &\lesssim  \sum_{\ell = \floor{q} + 1}^{\floor{2q}} d^{-\ell/2} \sum_{j = \floor{q} + 1}^{\ell} |\alpha_{j\ell}| \E\norm*{\parens*{\frac{\blr\blr^\top}{d}}^{\odot \ell} \odot \blDelta^{(j)}}\\
    &\leq \sum_{\ell = \floor{q} + 1}^{\floor{2q}} \sum_{j = \floor{q} + 1}^{\ell} |\alpha_{j\ell}| d^{-\ell/2} \E \norm*{\parens*{\frac{\blr}{\sqrt{d}}}^{\odot \ell}}_\infty^2 \E\norm*{\blDelta^{(j)}},
\end{align*}
where the last inequality follows from  Lemma \ref{lem:hadamard-outer} and the independence of $\blr$ and $\blDelta$.
Returning to the expression above, we have 
\[
\E \norm*{\parens*{\frac{\blr}{\sqrt{d}}}^{\odot \ell}}_\infty^2 = d^{-\ell} \E \max_{i} \norm{\blx_i}_2^{2\ell} \lesssim_{\log} 1.
\]
Moreover, by Corollary \ref{cor:Gnorm}, we have
\[
\E\norm{\blDelta^{(j)}} \lesssim_{\log} \sqrt{nd^{-j}} \asymp d^{q/2 - j/2}.
\]
Combining these bounds, we obtain
\[
\E\norm{\blR - \bar{\blR}} \lesssim_{\log} \sum_{\ell, j = \floor{q} + 1}^{\floor{2q}}  d^{-\ell/2} d^{j/2} d^{q/2} d^{-j/2} \lesssim \sum_{\ell = \floor{q} + 1}^{\floor{2q}} d^{q/2 - \ell/2} \lesssim d^{\frac{q - \floor{q} - 1}{2}}. 
\]
For the diagonal part of the error, we proceed similarly to the proof of Theorem \ref{thm:gaussianK}, conditioning on the same event $\scrE$:
\begin{align*}
    \norm{\diag \blK - \diag \bar{\blK}} &= \max_i \abs*{f\parens*{\frac{\normt{\blx_i}^2}{d}} - \sum_{\ell = 0}^{\floor{q}}\frac{f^{(j)}(0)}{j!} \parens*{\frac{r_i^2}{d}}^\ell  - \sum_{\ell = \floor{q}+ 1}^{\floor{2q}}\frac{f^{(j)}(0)}{j!} \parens*{\frac{r_i^2}{d}}^\ell d^{-\ell/2} \sum_{k=0}^{\floor{q}} \alpha_{k\ell} - f(1) + \sum_{j=0}^{\floor{q}}\frac{f^{(j)}(0)}{j!}}\\
    &\lesssim_{\log} L d^{-1/2} + \sum_{\ell = 0}^{\floor{q}} \max_i \abs*{\parens*{\frac{r_i^2}{d}}^\ell - 1} + \max_i \sum_{\ell = \floor{q} +1}^{\floor{2q}} d^{-\ell/2} \parens*{\frac{r_i^2}{d}}^\ell\sum_{k=0}^{\floor{q}}|\alpha_{k\ell}|\\
    &\lesssim d^{-1/2} + d^{-1/2} + d^{-\floor{q}/2-1/2 + \floor{q}/2}\\
    &\lesssim d^{-1/2}.
\end{align*}
Combining the bounds on the off-diagonal and diagonal components, we can conclude that $\norm{\blK - \bar{\blK}} \to 0$ with probability tending to $1$ as $n,d \to \infty$. 
\qed

\section{Proof of Theorem \ref{thm:biaslowerbound}}\label{app:biasproof}
We consider i.i.d. samples drawn from the model
\[
y_i = g^*(\blx_i) + \epsilon_i,
\]
where $\blx_1,\ldots,\blx_n \text{ i.i.d.} \sim \scrN(\mathbf{0}, \blSigma)$, and $\epsilon_1,\ldots,\epsilon_n \text{ i.i.d} \sim \scrN(0,\sigma^2)$. We will assume that there exists some integer $L>4q-2$ such that $f^{(L+1)}$ is uniformly bounded by a constant. Moreover, we will consider $g^* \in L^2(\scrN(\mathbf{0}, \blSigma))$ of the form
\begin{align}
g^*(\blx) = \sum_{k=0}^K c_k g_k(\ip{\blx}{\blSigma^{-1/2}\blv_k}),
\end{align}
for some constant $K$, fixed unit vectors $\blv_k \in \scrS^{d-1}$, and functions $g_k\in L^2(\scrN(0,1))$.

We set up the following basic notation:
\begin{itemize}
    \item \emph{Data matrix, label and noise vector:} As is standard, we denote $\blX = \begin{bmatrix} \blx_1 & \blx_2 & \ldots & \blx_n \end{bmatrix}^\top$, $\blY = \begin{bmatrix} y_1 & y_2 & \ldots & y_n \end{bmatrix}^\top$ and $\blepsilon = \begin{bmatrix} \epsilon_1 & \epsilon_2 & \ldots & \epsilon_n \end{bmatrix}^\top$ as the data matrix, label and noise vector, respectively.
    \item \emph{Function evaluation vector:} We write the function evaluation vector as \newline $\blg = \begin{bmatrix} g^*(\blx_1) & g^*(\blx_2) & \ldots & g^*(\blx_n) \end{bmatrix}$.
    \item \emph{Empirical kernel matrix and vector:} We denote the empirical kernel matrix by $\blK$ where $K_{ij} = k(\blx_i,\blx_j)$. Further, we denote the vector $\blV = \begin{bmatrix} V_1 & V_2 & \ldots & V_n \end{bmatrix}^\top$ where $V_i = \E_{\blx} \left[g^*(\blx) k(\blx, \blx_i) \right]$.
    \item \emph{Correlation matrix:} We define the correlation matrix that appears in the analysis by $\blM$ where $M_{ij} = \E_{\blx} \left[k(\blx, \blx_i) \cdot k(\blx, \blx_j)\right]$. Note that this is exactly the correlation matrix appearing in Theorem \ref{thm:gen-bound}, applied here to the original inner-product kernel $k(\blx, \bly) = f\parens*{\frac{\ip{\blx}{\bly}}{\tau}}$.
\end{itemize}

With this notation, the expression for the test bias of KRR can be written in closed form as follows:
\begin{align}\label{eq:MSE-KRR-initial-expression}
    \mathsf{Bias}(\hat{g},g^*) := \|\E_y\hat{g} - g^*\|_{L^2}^2 &= \norm{g^*}_{L^2}^2 - 2 \blg^\top (\blK + \lambda \blI_n)^{-1} \blV + \blg^\top(\blK + \lambda \blI_n)^{-1} \blM (\blK + \lambda \blI_n)^{-1} \blg.
\end{align}
The analysis proceeds by approximating each of the latter two terms in this expansion with versions corresponding to a low-degree polynomial function, which will allow us to conclude that the bias of $\hat{g}$ is well-approximated by the bias of a low-degree polynomial. 

We begin by approximating the matrix $\blM$ and the vector $\blV$ separately. We will condition on the event $\tilde{\scrE}$, which holds with probability at least $1-\frac{c}{n^2}$:
\begin{align}\label{eq:good-event-2}
\tilde{\scrE} \coloneqq \braces*{\blX \colon \max_{1\leq i, j \leq n} \abs*{\frac{\ip{\blSigma^{1/2}\blx_i}{\blSigma^{1/2}\blx_j}}{\tau_2} - \delta_{ij}} \lesssim \tau_2^{-1} \tau_4^{1/2} \log{n}}
\end{align}
Furthermore, by assumption on $g^*$, we can also condition on the event that
\[
\norm{\blg}_2^2 \lesssim n\log{n},
\]
which holds with probability tending to $1$ by Markov's inequality and  because $g^* \in L^2$.

\subsection{Approximation of the $\blM$ matrix}
First, consider a fixed $\blx$ and perform a Taylor expansion of the kernel function to get
\begin{align*}
    k(\blx_i, \blx) = \sum_{\ell=0}^L \frac{f^{(\ell)}(0)}{\ell! \tau^\ell} \ip{\blx}{\blx_i}^\ell + \frac{f^{(L+1)}(\zeta_i)}{(L+1)!\tau^{L+1}}\ip{\blx_i}{\blx}^{L+1},
\end{align*}
where $\zeta_i$ is between $0$ and $\frac{1}{\tau}\ip{\blx_i}{\blx}$. Defining $\blz \coloneqq \blSigma^{-1/2}\blx$,  $r_i \coloneqq \normt{\blSigma^{1/2} \blx_i}$ and $\blu_i \coloneqq \frac{\blSigma^{1/2}\blx_i}{\normt{\blSigma^{1/2} \blx_i}}$, we can write
\begin{align*}
    k(\blx_i, \blx) = \sum_{\ell=0}^L \frac{f^{(\ell)}(0)r_i^\ell}{\ell! \tau^\ell} \ip{\blz}{\blu_i}^\ell + \frac{f^{(L+1)}(\zeta_i)}{(L+1)!\tau^{L+1}}\ip{\blx_i}{\blx}^{L+1}.
\end{align*}
We rewrite the first term in terms of a univariate Hermite expansion to obtain
\begin{align}\label{eq:kdecomp}
    k(\blx_i, \blx) = \sum_{\ell=0}^L b_{\ell,i} \He_\ell(\ip{\blz}{\blu_i}) + \frac{f^{(L+1)}(\zeta_i)}{(L+1)!\tau^{L+1}}\ip{\blx_i}{\blx}^{L+1},
\end{align}
where
\[
b_{\ell,i} \coloneqq \frac{1}{\ell!}\sum_{m=\ell}^L \frac{f^{(m)}(0)r_i^m}{m!\tau^m}\E_{z \sim \scrN(0,1)}[z^m \He_\ell(z)].
\]
It is easy to verify that, on the event $\tilde{\scrE}$, these coefficients are bounded as $|b_{\ell,i}|\lesssim_{\log} \tau_1^{-\ell}\tau_2^{\ell/2}$. 
Using the decomposition in Equation \eqref{eq:kdecomp} and Lemma \ref{lem:unitredux}, we can write
\begin{align*}
    M_{ij} = \underbrace{\sum_{\ell=0}^L \ell! b_{\ell,i}b_{\ell,j}\ip{\blu_i}{\blu_j}^\ell}_{M_{ij}^{(1)}} &+ \underbrace{\sum_{\ell=0}^L b_{\ell,i}\E_{\blx}\brackets*{\He_\ell(\ip{\blz}{\blu_i})\frac{f^{(L+1)}(\zeta_j)}{(L+1)!\tau^{L+1}}\ip{\blx_j}{\blx}^{L+1}}}_{M_{ij}^{(2)}} \\
    &+ \underbrace{\sum_{\ell=0}^L b_{\ell,j}\E_{\blx}\brackets*{\He_\ell(\ip{\blz}{\blu_j})\frac{f^{(L+1)}(\zeta_i)}{(L+1)!\tau^{L+1}}\ip{\blx_i}{\blx}^{L+1}}}_{M_{ij}^{(3)}} \\
    &+\underbrace{\E_{\blx}\brackets*{\frac{f^{(L+1)}(\zeta_i)}{(L+1)!\tau^{L+1}}\ip{\blx_i}{\blx}^{L+1}\frac{f^{(L+1)}(\zeta_j)}{(L+1)!\tau^{L+1}}\ip{\blx_j}{\blx}^{L+1}}}_{M_{ij}^{(4)}} \\
\end{align*}
For the latter three terms, we use the Cauchy-Schwarz inequality and event $\tilde{\scrE}$ to obtain the following bounds (using the cruder bound $|b_{\ell,i}| \lesssim_{\log} \tau_1^{-\ell} \tau_2^{\ell/2} \lesssim `\tau_1^{-\ell/2}$):
\begin{align*}
\abs*{M^{(2)}_{ij}} &\lesssim \sum_{\ell=0}^L \tau^{-\ell/2} \tau^{-L-1}\normt{\blSigma^{1/2} \blx_i}^{L+1} \lesssim_{\log} \tau^{\frac{-L-1}{2}},\\
\abs*{M^{(3)}_{ij}} &\lesssim \sum_{\ell=0}^L \tau^{-\ell/2} \tau^{-L-1}\normt{\blSigma^{1/2}\blx_j}^{L+1} \lesssim_{\log} \tau^{\frac{-L-1}{2}},\\
\abs*{M^{(4)}_{ij}} &\lesssim  \tau^{-L-1}\normt{\blSigma^{1/2}\blx_i}^{L+1}  \tau^{-L-1}\normt{\blSigma^{1/2}\blx_j}^{L+1} \lesssim_{\log} \tau^{-L-1},
\end{align*}
where we use the fact that we have conditioned on $\tilde{\scrE}$ and $\tau_2 \leq \tau$.
Hence, the corresponding matrices have operator norm bounded up to log factors by $n\tau^{\frac{-L-1}{2}} = \tau^{q-\frac{L}{2}-\frac{1}{2}}$. 
Next, consider the term 
\[\sum_{\ell=0}^L \ell! b_{\ell,i}b_{\ell,j}\ip{\blu_i}{\blu_j}^\ell =\colon \sum_{\ell=0}^L M_{ij}^{(1,\ell)}.
\]
Bounding each term in the summation separately, we obtain for $i \neq j$ (again, using the event $\tilde{\scrE}$), we have
\[
\abs*{M^{(1,\ell)}_{ij}} \lesssim  \tau_1^{-2\ell}\tau_2^{\ell} |\ip{\blu_i}{\blu_j}|^\ell \lesssim_{\log} \tau_1^{-2\ell}\tau_2^{\ell} \tau_2^{-\ell}\tau_4^{\ell/2} = \tau_1^{-2\ell}\tau_4^{\ell/2} \lesssim \tau_1^{-3\ell/2}.
\]
Now, for any $\ell > \floor{4q/3}$, we can use the triangle inequality to upper bound the operator norm of $\norm*{\blM^{(1,\ell)}}$ by the sum of the norm of the diagonal part (i.e., the maximum absolute diagonal entry) and the norm of the off-diagonal part (for which we use a simple Frobenius norm bound).
In more detail, we have
\begin{align*}
\norm*{\blM^{(1,\ell)}} &\lesssim_{\log}  n \tau^{-3\ell/2} + \max_{1 \leq i \leq n} \abs{M_{ii}^{(1,\ell)}}\\
&\lesssim n \tau^{-3\ell/2} + \max_{1 \leq i \leq n} \sum_{\ell = 0}^L |b_{\ell,i}|^2\\
&\lesssim \tau^{q-3\ell/2} + \sum_{\ell = 0}^L \tau_1^{-2\ell} \tau_2^{\ell}\\
&\lesssim \tau^{q-3\ell/2}.
\end{align*}
Finally, defining
\[
\bar{M}_{ij} \coloneqq \sum_{\ell=0}^{\floor{4q/3}} \ell! b_{\ell,i}b_{\ell,j}  \ip{\blu_i}{\blu_j}^\ell,
\]
we can conclude that
\begin{align*}
     \abs*{\blg^\top(\blK + \lambda \blI_n)^{-1} (\blM -\bar{\blM})(\blK + \lambda \blI_n)^{-1} \blg}&\lesssim \normt{\blg}^2 \norm{(\blK + \lambda \blI_n)^{-1}}^2 \norm{\blM - \bar{\blM}}\\
     &\lesssim_{\log} n (\tau^{q-\frac{L}{2}-\frac{1}{2}} +  \tau^{q-\frac{3}{2}(\floor{\frac{4q}{3}} + 1)})\\
     &= \tau^{2q-\frac{L}{2}-\frac{1}{2}} +\tau^{2q-\frac{3}{2}(\floor{\frac{4q}{3}} + 1)}\\
     &= o_\tau(1).
\end{align*}
The second line above uses the fact that $\blK + \lambda \blI$ has eigenvalues larger than a constant. In the case $\lambda = 0$, this is guaranteed with probability tending to $1$ by the approximation result for $\blK$ in the main paper (note this only requires the more crude bound obtained by \cite{donhauser2021rotational}. More precisely, by Weyl's inequality and Theorem \ref{thm:gaussianK}, we have
\[
\mu_n(\blK) \geq \mu_n(\bar{\blK}) - o_\tau(1) \geq \parens*{f(1) - \sum_{j=0}^{\floor{4q/3}}\frac{f^{(j)}(0)}{j!}} - o_\tau(1) \geq c,
\]
for some $c>0$ and sufficiently large $\tau$. We note that we use the approximation from Theorem 3, but for this result one could also use the more crude approximation in~\cite{donhauser2021rotational}.

\subsection{Approximation of the $\blV$ vector}

We approximate this term in a similar manner. We can write the $i$-th entry of $\blV$ as
\begin{align*}
    V_i &= \underbrace{\sum_{\ell=0}^L b_{\ell,i} \E_{\blx}\brackets*{\He_{\ell}(\ip{\blz}{\blu_i})g^*(\blx)}}_{V_i^{(1)}} + \underbrace{\E_{\blx}\brackets*{\frac{f^{(L+1)}(\zeta_i)}{(L+1)!\tau^{L+1}}\ip{\blx}{\blx_i}^{L+1}g^*(\blx)}}_{V_i^{(2)}}.\\
\end{align*}
We use the Cauchy-Schwarz inequality and event $\tilde{\scrE}$ to obtain the bound
\[
\abs*{V_i^{(2)}} \lesssim_{\log} \tau^{-L-1}\tau_2^{(L+1)/2} \lesssim \tau^{\frac{-L-1}{2}},
\]
from which we can conclude $\normt{\blV^{(2)}} \lesssim_{\log} \sqrt{n}\tau^{\frac{-L-1}{2}} = \tau^{\frac{q-L-1}{2}}$.
For $\blV^{(1)}$, consider each term separately. 
\[
\abs*{V_i^{(1,\ell)}} \coloneqq b_{\ell,i}\E_{\blx}\brackets*{\He_\ell(\ip{\blz}{\blu_i}g^*(\blx)} \lesssim \tau^{-\ell}\tau_2^{\ell/2} \abs*{\E_{\blx}\brackets*{\He_\ell(\ip{\blz}{\blu_i}g^*(\blx)}}.
\]
Using the assumed form of $g^*$ in Equation~\eqref{eq:gstar}, we can compute this expectation as
\begin{align*}
\abs{\E_{\blx}\brackets*{\He_\ell(\ip{\blz}{\blu_i}g^*(\blx)}} &= \abs*{\sum_{k=0}^K c_k \E_{\blz}[\He_\ell(\ip{\blz}{\blu_i})g_k(\ip{\blz}{\blv_k}]}\\
 &= \abs*{\sum_{k=0}^K c_k \sum_{j=0}^\infty \alpha_{jk}\E_{\blz}[\He_\ell(\ip{\blz}{\blu_i})\He_j(\ip{\blz}{\blv_k})]}\\
&= \abs*{\sum_{k=0}^K \ell! c_k \alpha_{\ell,k} \ip{\blu_i}{\blv_k}^\ell}\\
&\lesssim \sum_{k=0}^K \abs{\ip{\blu_i}{\blv_k}}^\ell\\
&= \sum_{k=0}^K r_i^{-\ell} \abs{\ip{\blSigma^{1/2}\blx_i}{\blv_k}}^\ell\\
&= \sum_{k=0}^K r_i^{-\ell} \abs{\ip{\blz_i}{\blSigma \blv_k}}^\ell,
\end{align*}
where $\alpha_{jk}$ are the Hermite coefficients of $g_k$. Recall also the $\blv_k$ are fixed unit vectors. Note that for any $k$, we have
\[
\abs{\ip{\blz_i}{\blSigma \blv_k}} \overset{d}{=} \norm{\blSigma \blv_k} \abs{z_k} \leq \abs{z_k},
\]
for a standard normal variable $z_k$, so via a standard Gaussian tail bound and a union bound over all $K$ variables, we have $\abs{z_k} \leq \sqrt{2\log{n}}$ for all $k$, with probability at least $1-\frac{K}{n}$. So, we can conclude that 
\[
\abs{\E_{\blx}\brackets*{\He_\ell(\ip{\blz}{\blu_i}g^*(\blx)}} \lesssim_{\log} r_i^{-\ell} \lesssim_{\log} \tau_2^{-\ell/2},
\]
with probability tending to $1$. From this, we have
\[
\normt*{V_i^{(1,\ell)}} \lesssim_{\log} \sqrt{n}\tau^{-\ell}\tau_2^{\ell/2}\tau_2^{-\ell/2} \asymp \tau^{\frac{q}{2}-\ell} 
\]
Combining the above, we can define $\bar{\blV} \coloneqq \sum_{\ell=0}^{\floor{4q/3}}\blV^{(1, \ell)}$. So, we have

\begin{align*}
  \abs*{\blg^\top(\blK+\lambda \blI_n)^{-1}(\blV - \bar{\blV})} &\lesssim_{\log} \sqrt{n} \tau^{\frac{q}{2} - \floor{\frac{4q}{3}} - 1} = \tau^{q - \floor{\frac{4q}{3}} - 1} = o_\tau(1).
\end{align*}
(Note that we could have actually used the sharper approximation $\bar{\blV} := \sum_{\ell = 0}^{\floor{q}} \blV^{(1, \ell)}$ for this part of the proof.
However, we pick the degree-$\floor{4q/3}$ approximation to match the approximation of the $\blM$ term for convenience.)

\subsection{Concluding the argument}
Motivated by the results from the previous two sections, we can define the following function (which depends on $\blx_1, \dots, \blx_n$):
\[
\bar{g}(\blx) \coloneqq \blg^\top (\blK + \lambda \blI_n)^{-1} \bar{k}(\blX, \blx),
\]
where $\bar{k}(\blX, \blx) \in \R^n$ is a vector with $i$-th entry given by 
\[
\bar{k}(\blx_i, \blx) \coloneqq \sum_{\ell=0}^{\floor{4q/3}} b_{\ell,i}\He_\ell(\ip{\blz}{\blu_i}) = \sum_{\ell=0}^{\floor{4q/3}} b_{\ell,i}\He_\ell(\ip{\blx}{\blSigma^{-1/2}\blu_i}).
\]

Note that $\bar{g}$ is a polynomial of $\blx$ of degree at most $\floor{\frac{4q}{3}}$, so its bias is lower bounded by the bias of the best $\floor{\frac{4q}{3}}$ approximation to $g^*$. Moreover, we have

\[
\mathsf{Bias}(\bar{g},g^*) = \norm{g^*}_{L^2}^2 - 2 \blg^\top (\blK + \lambda \blI_n)^{-1} \bar{\blV} + \blg^\top(\blK + \lambda \blI_n)^{-1} \bar{\blM} (\blK + \lambda \blI_n)^{-1} \blg
\]
The results of the previous two sections imply that
\[
\abs*{\mathsf{Bias}(\hat{g},g^*)- \mathsf{Bias}(\bar{g},g^*)} = o_\tau(1).
\]
Therefore, we can conclude that 
\[
\mathsf{Bias}(\hat{g}, g^*) \geq \inf_{p \in \scrP_{\leq \floor{\frac{4q}{3}}}} \norm{p - g^*}_{L^2}^2  - o_\tau(1).
\]
\qed
\end{document}

%% file: defs/definitions.tex
\usepackage{amsmath,amssymb,amsfonts,amsthm,optidef}
\usepackage{cite}
\usepackage{hyperref}
\usepackage{bm}
\usepackage{bbm}
\usepackage{xparse}
\usepackage{import}
\usepackage{xcolor}
\usepackage[normalem]{ulem}

\makeatletter
\@ifpackageloaded{unicode-math}{%
	\newcommand{\mymathbold}{\symbf}%
}{%
	\usepackage{bm}%
	\newcommand{\mymathbold}{\bm}%
}
\makeatother

\subimport{defs}{symbols.tex}

\newtheorem{proposition}{Proposition}
\newtheorem{theorem}{Theorem}

\newtheorem{lemma}{Lemma}
\newtheorem{corollary}{Corollary}

\DeclareMathOperator{\E}{\mathbb{E}}
\renewcommand{\P}{\operatorname{\mathbb{P}}}

\newcommand{\tr}{\operatorname{tr}}

\newcommand{\diag}{\operatorname{diag}}

\DeclarePairedDelimiter{\norm}{\lVert}{\rVert}
\DeclarePairedDelimiter{\abs}{\lvert}{\rvert}
\DeclarePairedDelimiter{\floor}{\lfloor}{\rfloor}

\DeclarePairedDelimiter{\braces}{\{}{\}}
\DeclarePairedDelimiter{\parens}{(}{)}
\DeclarePairedDelimiter{\brackets}{[}{]}
\DeclarePairedDelimiterX{\ip}[2]{\langle}{\rangle}{#1,#2}

\DeclarePairedDelimiterXPP{\normsub}[2]{}{\lVert}{\rVert}{_{#2}}{#1}
\DeclarePairedDelimiterXPP{\ipsub}[3]{}{\langle}{\rangle}{_{#3}}{#1,#2}

\DeclarePairedDelimiterXPP{\ipHS}[2]{}{\langle}{\rangle}{_{\mathrm{HS}}}{#1, #2}
\DeclarePairedDelimiterXPP{\normHS}[1]{}{\lVert}{\rVert}{_{\mathrm{HS}}}{#1}

\DeclarePairedDelimiterXPP{\ipF}[2]{}{\langle}{\rangle}{_{\mathrm{F}}}{#1, #2}
\DeclarePairedDelimiterXPP{\normF}[1]{}{\lVert}{\rVert}{_{\mathrm{F}}}{#1}

\DeclarePairedDelimiterXPP{\normt}[1]{}{\lVert}{\rVert}{_{2}}{#1}
\DeclarePairedDelimiterXPP{\normo}[1]{}{\lVert}{\rVert}{_{\mathrm{1}}}{#1}

\DeclarePairedDelimiterXPP{\dkl}[2]{\operatorname{D_{KL}}}{(}{)}{}{#1 \: \delimsize\Vert \: #2}

\DeclarePairedDelimiterXPP{\restr}[2]{}{{}}{\vert}{_{#2}}{#1}

\newcommand{\R}{\mathbb{R}}

\newcommand{\x}{\bm{x}}

\newcommand{\He}{\text{He}}

\newcommand{\given}{\:\vert\:}

\newcommand{\simiid}{\overset{\mathclap{\text{i.i.d.}}}{\sim}}


%% file: defs/symbols.tex


\newcommand{\scrbar}[1]{\overline{\mathcal{#1}}}

\ExplSyntaxOn

\cs_new_protected:Nn \bamboo_define:nnnnnN
{
	\cs_new_protected:cpx { #3 #1 #4 } { \exp_not:N #5{#6{#2}} }
}

\int_step_inline:nnn { `A } { `Z }
{
	\bamboo_define:nnnnnN
	{ \char_generate:nn { #1 } { 11 } }
	{ \char_generate:nn { #1 } { 11 } }
	{ bl }
	{ }
	{ \mymathbold }
	\use:n
	
	
	\bamboo_define:nnnnnN
	{ \char_generate:nn { #1 } { 11 } }
	{ \char_generate:nn { #1 } { 11 } }
	{ scr }
	{ }
	{ \mathcal }
	\use:n
	
	\bamboo_define:nnnnnN
	{ \char_generate:nn { #1 } { 11 } }
	{ \char_generate:nn { #1 } { 11 } }
	{ }
	{ hat }
	{ \widehat }
	\use:n
	
	\bamboo_define:nnnnnN
	{ \char_generate:nn { #1 } { 11 } }
	{ \char_generate:nn { #1 } { 11 } }
	{ }
	{ br }
	{ \overline }
	\use:n
	
	\bamboo_define:nnnnnN
	{ \char_generate:nn { #1 } { 11 } }
	{ \char_generate:nn { #1 } { 11 } }
	{ }
	{ tl }
	{ \widetilde }
	\use:n
	
	\bamboo_define:nnnnnN
	{ \char_generate:nn { #1 } { 11 } }
	{ \char_generate:nn { #1 } { 11 } }
	{ scr }
	{ br }
	{ \scrbar }
	\use:n
}
\int_step_inline:nnn { `a } { `z }
{
	\bamboo_define:nnnnnN
	{ \char_generate:nn { #1 } { 11 } }
	{ \char_generate:nn { #1 } { 11 } }
	{ bl }
	{ }
	{ \mymathbold }
	\use:n
	
	\bamboo_define:nnnnnN
	{ \char_generate:nn { #1 } { 11 } }
	{ \char_generate:nn { #1 } { 11 } }
	{ }
	{ hat }
	{ \hat }
	\use:n
	
	\bamboo_define:nnnnnN
	{ \char_generate:nn { #1 } { 11 } }
	{ \char_generate:nn { #1 } { 11 } }
	{ }
	{ br }
	{ \bar }
	\use:n
	
	\bamboo_define:nnnnnN
	{ \char_generate:nn { #1 } { 11 } }
	{ \char_generate:nn { #1 } { 11 } }
	{ }
	{ tl }
	{ \tilde }
	\use:n                            
}
\clist_map_inline:nn
{
	Gamma,Delta,Theta,Lambda,Xi,Pi,Sigma,Phi,Psi,Omega
}
{
	\bamboo_define:nnnnnN
	{ #1 }
	{ #1 }
	{ bl }
	{ }
	{ \mymathbold }
	\use:c
	
	\bamboo_define:nnnnnN
	{ #1 }
	{ #1 }
	{ op }
	{ }
	{ \op }
	\use:c
	
	\bamboo_define:nnnnnN
	{ #1 }
	{ #1 }
	{ scr }
	{ }
	{ \mathcal }
	\use:c
	
	\bamboo_define:nnnnnN
	{ #1 }
	{ #1 }
	{ }
	{ hat }
	{ \widehat }
	\use:c
	
	\bamboo_define:nnnnnN
	{ #1 }
	{ #1 }
	{ }
	{ br }
	{ \overline }
	\use:c
	
	\bamboo_define:nnnnnN
	{ #1 }
	{ #1 }
	{ }
	{ tl }
	{ \widetilde }
	\use:c
}
\clist_map_inline:nn
{
	alpha,beta,gamma,delta,epsilon,zeta,eta,theta,iota,kappa,
	lambda,mu,nu,xi,pi,rho,sigma,tau,phi,chi,psi,omega
}
{
	\bamboo_define:nnnnnN
	{ #1 }
	{ #1 }
	{ bl }
	{ }
	{ \mymathbold }
	\use:c
	
	\bamboo_define:nnnnnN
	{ #1 }
	{ #1 }
	{ }
	{ hat }
	{ \hat }
	\use:c
	
	\bamboo_define:nnnnnN
	{ #1 }
	{ #1 }
	{ }
	{ br }
	{ \bar }
	\use:c
	
	\bamboo_define:nnnnnN
	{ #1 }
	{ #1 }
	{ }
	{ tl }
	{ \tilde }
	\use:c
}

\ExplSyntaxOff